\documentclass{article}

\usepackage[preprint]{neurips_2025}

\usepackage[utf8]{inputenc} 
\usepackage[T1]{fontenc}    
\usepackage{hyperref}       
\usepackage{url}            
\usepackage{booktabs}       
\usepackage{amsfonts}       
\usepackage{nicefrac}       
\usepackage{microtype}      
\usepackage{xcolor}         

\usepackage{adjustbox} 
\usepackage{subcaption} 
\usepackage{stfloats}   

\usepackage{algorithm}
\usepackage{algorithmic}

\usepackage{amsmath}
\usepackage{amssymb}
\usepackage{mathtools}
\usepackage{amsthm}
\usepackage{multirow}

\usepackage{cleveref}

\theoremstyle{plain}
\newtheorem{theorem}{Theorem}[section]
\newtheorem{proposition}[theorem]{Proposition}

\theoremstyle{definition}
\newtheorem{definition}[theorem]{Definition}

\newtheorem{application}[theorem]{Application}
\theoremstyle{remark}

\title{Existing LLMs Are Not Self-Consistent\\
For Simple Tasks}

\author{%
  \vspace{2.5mm}
  \textbf{Zhenru Lin}$^{1}$ \hspace{6mm} 
  \textbf{Jiawen Tao}$^{1}$ \hspace{6mm} 
  \textbf{Yang Yuan}$^{1,2,3}$\textsuperscript{*} \hspace{6mm} 
  \textbf{Andrew Chi-Chih Yao}$^{1,2,3}$\thanks{Corresponding Author} \\
  $^{1}$IIIS, Tsinghua University \\
  $^{2}$Shanghai AI Laboratory, Shanghai, China \\
  $^{3}$Shanghai Qizhi Institute \\
  \texttt{\{lzr22,taojw23\}@mails.tsinghua.edu.cn} \\
  \texttt{\{yuanyang,andrewcyao\}@tsinghua.edu.cn} \\
}

\begin{document}

\maketitle

\begin{abstract}
Large Language Models (LLMs) have grown increasingly powerful, yet ensuring their decisions remain transparent and trustworthy requires self-consistency --- no contradictions in their internal reasoning. Our study reveals that even on simple tasks, such as comparing points on a line or a plane, or reasoning in a family tree, all smaller models are highly inconsistent, and even state-of-the-art models like DeepSeek-R1 and GPT-o4-mini are not fully self-consistent. To quantify and mitigate these inconsistencies, we introduce inconsistency metrics and propose two automated methods—a graph-based and an energy-based approach. While these fixes provide partial improvements, they also highlight the complexity and importance of self-consistency in building more reliable and interpretable AI. The code and data are available at \url{https://github.com/scorpio-nova/llm-self-consistency}.
\end{abstract}

\section{Introduction}
\label{intro}


Large Language Models (LLMs) have demonstrated remarkable capabilities, solving diverse problems and assisting humans in making informed decisions. However, if we want these decisions to be interpretable and well-justified, LLMs must satisfy at least one fundamental requirement: self-consistency --- the ability to maintain coherent internal reasoning without contradictions.

Consider three points --- A, B, and C --- on a one-dimensional line. Suppose an LLM attempts to reason about their relative positions and asserts that A is to the left of B, and B is to the left of C. However, it also claims that C is between A and B. This creates an obvious contradiction, as such a configuration is logically impossible. If a model arrives at decisions based on incorrect or internally inconsistent reasoning, how can we trust the validity of its conclusions? 

From this perspective, if we want to trust an LLM’s reasoning, we must first ensure that it forms a systematic, complete, and self-consistent understanding of the relationships between relevant entities. For example, any multi-hop reasoning on a structured family tree should be consistent as in Figure~\ref{fig:family-tree}. 
Without such coherence, even correct answers may be unstable, and explanations may lack credibility. This raises a fundamental question: 

\begin{center}
	\fbox{
		\parbox{2.75in}{
			\textbf{How to measure self-consistency in LLMs?
	}}}
\end{center}

\begin{figure}[ht]
    \centering
    \includegraphics[width=\linewidth]{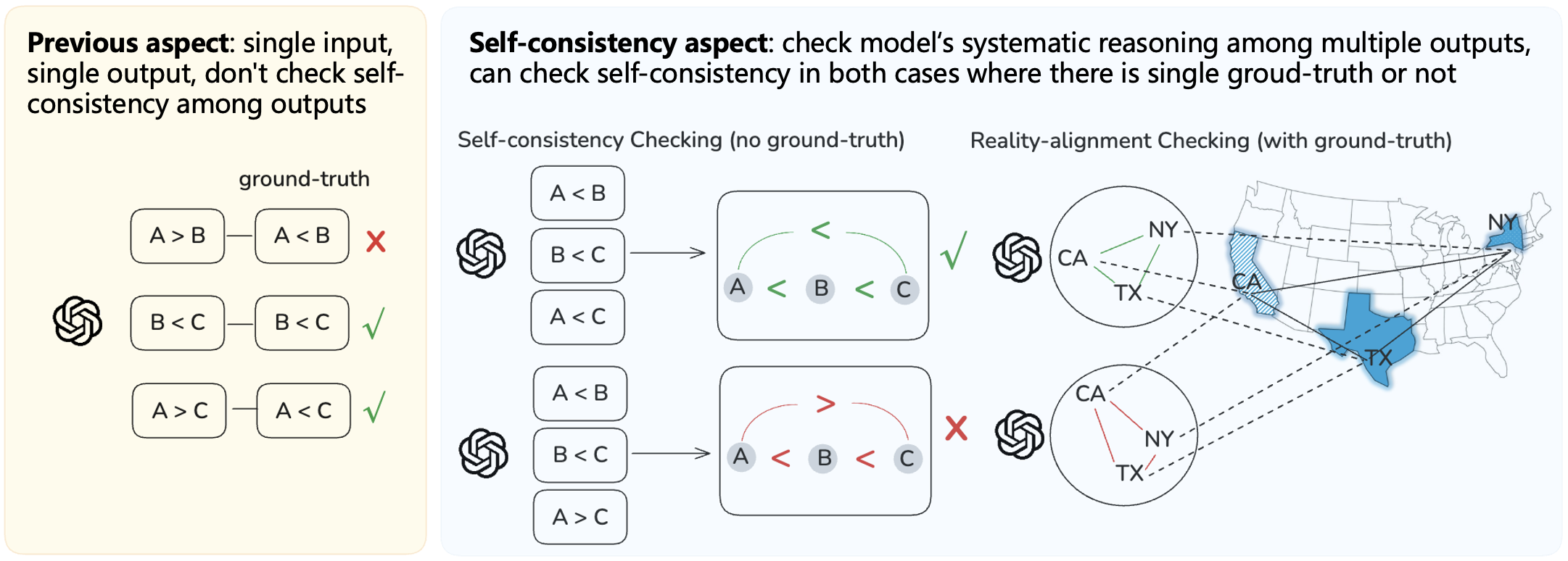}
    \caption{"Self-consistency" and "reality-alignment" are two fundamental prerequisites for LLM interpretability. Self-consistency ensures that relationships between objects remain logically coherent --- if the model asserts that $A<B<C$, then it must also infer that $A<C$. Reality-alignment requires that the relationships and objects perceived by the LLM accurately reflect real-world scenarios. For instance, the model's understanding of the relative positions of California (CA), New York (NY), and Texas (TX) should correspond to their actual geographic locations. Both requirements have categorical interpretations.}
    \label{fig:two-step interpretability}
\end{figure}

\begin{figure}
    \centering
    \includegraphics[width=0.7\linewidth]{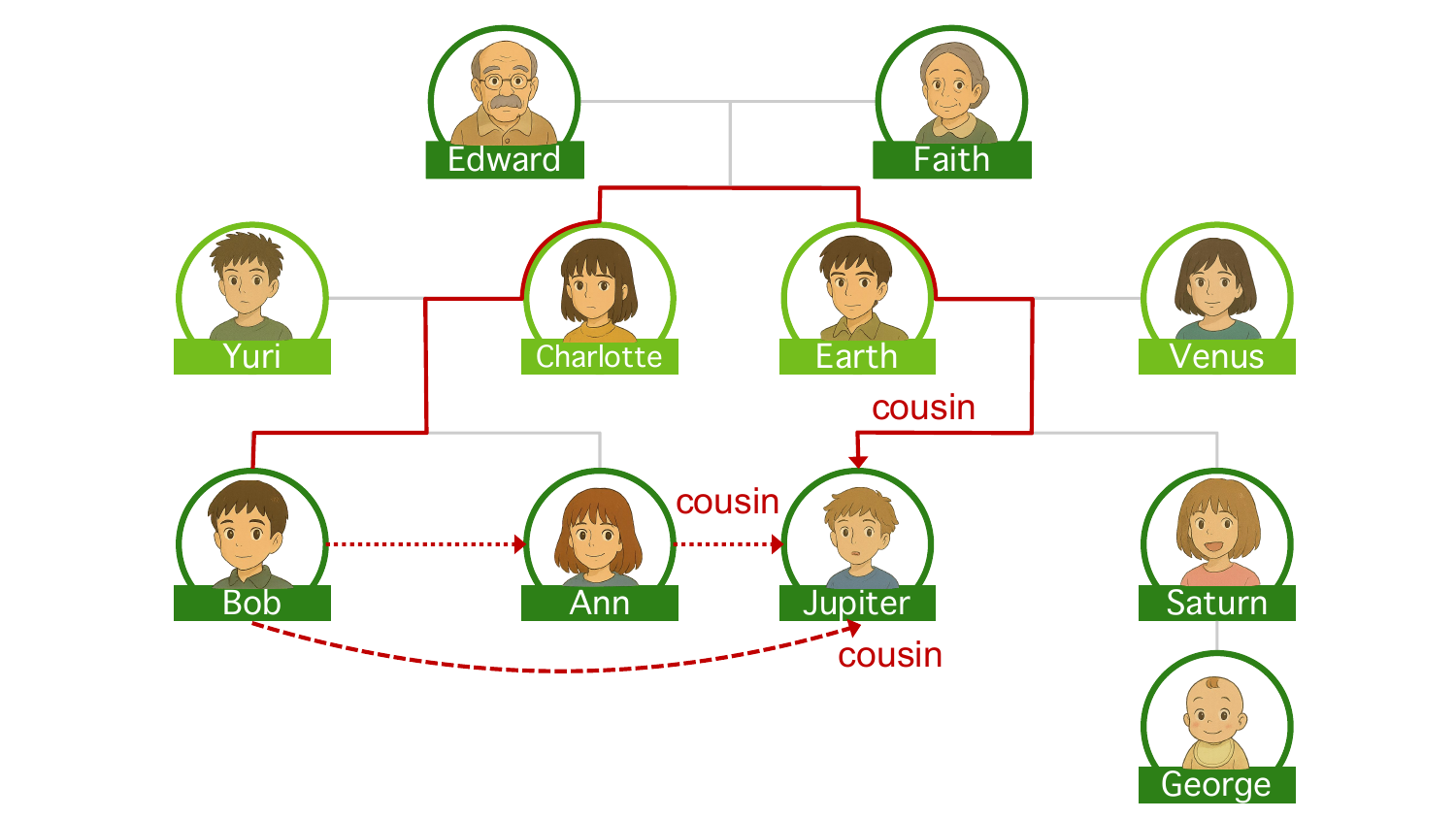}
    \caption{If the model is self-consistent, the multi-hop reasoning results of ``Bob'' and ``Jupiter'' on the family tree should always be the same using different reasoning paths.}
    \label{fig:family-tree}
\end{figure}

In this paper, we leverage a categorical perspective~\citep{Yuan23Power} to quantify the self-consistency of LLMs. To the best of our knowledge, this is the first attempt to systematically measure a model’s self-consistency\footnote{Previous works have proposed various definitions of self-consistency. Our approach introduces the first categorical-based definition of self-consistency, reflecting the model's systematic and internal understanding of the world and structures.} rather than merely assessing its performance on specific tasks.

As an initial attempt, we abstract temporal, spatial and kinship relationships as a total order relationship and propose a method for evaluating self-consistency. 
Our results on 7B to 16B models show high inconsistency rates, and even state-of-the-art models, including GPT-o4-mini and DeepSeek R1, show improved self-consistency but do not achieve complete self-consistency.

Moreover, self-consistency alone is not sufficient. A model could construct a meticulously coherent “world” in which every country on Earth is located at the North Pole and covers exactly one square meter --- internally consistent, but grossly misaligned with reality, which is useless. Using category-theoretic terminology, this means that the model’s internal category should be structurally preserved when mapped to the real-world category via a functor.
This brings us to another question:

\begin{center}
	\fbox{
		\parbox{2.3in}{
        \centering
			\textbf{
            How can LLMs be both \\self-consistent and reality-aligned?
	}}}
\end{center}

In this work, we propose two methods to address these challenges, yet neither achieves a perfect solution. Insights from category theory suggest that achieving full self-consistency may be inherently difficult for the current architecture of LLMs, 
because the next-token prediction framework naturally forms a forward-directed language category, lacking the necessary backward edges to enforce consistency. We hope our findings will underscore self-consistency as a fundamental challenge in LLM research and inspire further exploration toward models that can reason more coherently and reliably.

In summary, our key contributions are as follows:
\begin{itemize}
\item We proposed a new evaluation metric that reveals inconsistency in state-of-the-art LLMs.

\item We established a two-step requirement for interpretability: self-consistency for internal reasoning and reality-alignment for meaningful explanations.

\item We proposed two methods to improve self-consistency and reality-alignment.
\end{itemize}

\section{Background \& Motivation}
\label{sec:prelim}
\subsection{Self‑Consistency for Interpretability}
A model that contradicts itself cannot provide a stable explanation.
Hence interpretability requires \emph{(i) self‑consistency} and
\emph{(ii) alignment with reality}.  We treat consistency first and use
alignment only as a pilot study later.

\subsection{Category‑Theory View}
In particular, binary relations can be treated as morphisms in a small
category in which composition corresponds to transitivity; under this
perspective, \emph{morphism composition law} aligns exactly with
\emph{self-consistency}. Category theory is often considered abstract and hard to understand, so we keep the intuition here and defer full definitions to
Appendix~\ref{app:ct} for curious readers. If you are unfamiliar with the subject, feel free to skim or skip this part — the subsequent sections rely only on the concept of \emph{self-consistency}, not on category-theoretic machinery.

\subsection{Directed Graph}
Our inconsistency metric calculation is given under the definition of the directed graph. In our experiment for graph fixing algorithm, we use the notion of Tarjan's algorithm and simply-ordered graph.
\begin{definition}
\label{def:directed_graph}
A \textbf{directed graph} is a pair of sets $G = (V, E)$, where $V$ is a set of nodes, and $E \subseteq V \times V$ is a set of directed edges. Each edge in $E$ connects an ordered pair of nodes, denoted as $(u, v)$, where $u, v \in V$. 
\end{definition}

\begin{definition}
\label{def:scc}
A \textbf{strongly connected component (SCC)} in a directed graph $G = (V, E)$ is a maximal subgraph $G' = (V', E')$ such that for every pair of nodes $u, v \in V'$, there exists a directed path from $u$ to $v$ and a directed path from $v$ to $u$. In other words, every node in an SCC is reachable from every other node in the same SCC. 
\end{definition}

\begin{definition}
\label{def:tarjan_algorithm}
\textbf{Tarjan’s algorithm} \citep{tarjan} is a method for finding strongly connected components (SCCs) in a directed graph $G = (V, E)$.
\end{definition}

\begin{definition}
\label{def:simply_ordered_graph}
A directed graph $G = (V, E)$ is called a \textbf{simply-ordered graph} if there exists an order of $V$ such that 
    $E = \{$every edge goes from a lower-ordered node to higher-ordered node$\}$.
\end{definition}

\section{Inconsistency Score: A New Self-Consistency Metric}
\label{sec:metric}

Systematic reasoning requires that the composition of atomic binary relations
(e.g.\ ``before'', ``father'') never contradicts itself.
We test whether an LLM respects these constraints in three domains—
\textbf{time}, \textbf{space}, and \textbf{kinship}—whose primitive relations
and composition to check are summarised in Table \ref{tab:relations}. We check the necessary conditions for satisfying the composition law.

\begin{table}[h]
\centering
\begin{tabular}{@{}lll@{}}
\toprule
\textbf{Task} & \textbf{Primitive relations} & \textbf{Composition check} \\ \midrule
Time &
\textit{before}, \textit{after} &
$\textit{before}\,\circ\,\textit{before} = \textit{before}$ \\[2pt]

Space &
NE,\;NW,\;SE,\;SW &
$\text{NE}\,\circ\,\text{NE} = \text{NE},
  \text{NE}\,\circ\,\text{NW} = \text{NE or NW}$ \\[2pt]

Kinship &
father,\;mother,\;sister,\;son... &
$\text{father}\,\circ\,\text{father} = \text{grandfather},
  \text{sister}\,\circ\,\text{mother} = \text{aunt}$ \\ \bottomrule
\end{tabular}
\vspace{2mm}
\caption{Primitive binary relations and the compositions they must respect.}
\label{tab:relations}
\end{table}

Formally, let $\mathcal M$ denote all $N(N-1)$ ordered pairs among $N$ objects. Given an LLM prediction $A \subseteq \mathcal M$ and a trusted context $C \subseteq \mathcal M$, we define inconsistency as the minimum edit to any composition-consistent extension $B \supseteq C$:

\[
I(A; C) = \frac{\min_{B \in \mathcal S(C)} |A \setminus B|}{|\mathcal M|},\quad \mathcal S(C) = \{B \supseteq C \mid B \text{ is composition-consistent}\}.
\tag{*}
\]

\paragraph{From model outputs to graphs.}
For every ordered pair $(i,j)$ the LLM predicts a primitive relation.
We encode these answers as directed edges: For time (1D), a single graph $G_t=(V,E_t)$ where $(i\!\to\!j)\in E_t$ iff the model asserts $t_i<t_j$ ($i$ happens before $j$).
For space (2D), “A is \textsc{NE} of B’’ really says two things at once:\(x_A<x_B\) (east–west order) and \(y_A<y_B\) (north–south order). Thus, we construct two graphs: $G_x$ for the $x$-axis and $G_y$ for the $y$-axis.
For kinship, A graph $G_k=(V,E_k)$ whose edges follow the genealogy conventions of Table~\ref{tab:relations}.

\paragraph{Loops Indicate Inconsistency.}
If the model were perfectly consistent, we could order the nodes such that all edges point forward.
Inconsistency arises when the graph contains a \emph{cycle}—for example, A $\to$ B, B $\to$ C, C $\to$ A—indicating the model contradicts itself.

We can therefore quantify inconsistency by ``the proportion of edges that must be removed to eliminate all cycles''. However, computing the exact minimum feedback arc set is NP-complete \citep{Karp1972}, so we use approximations.

\paragraph{Picking a simply ordered graph and counting reverse edges.} 

Given an axis graph $G=(V,E)$ we can sorting based on the in-degree of node $i$ and obtain the order
$\pi:V\!\to\!\{1,\dots,N\}$ (Detailed algorithm and the theoretical guarantee can be found in Appendix~\ref{app:graph}). Then reverse edge is defined as the edges that contradict this order ($\pi(i)>\pi(j)$).

(i) When $C=\varnothing$, this reduces to classical “edit-to-consistency”.Let $E_{\mathrm{rev}}$ be the set of reverse edges.
Since there are at most $\tfrac12 N(N{-}1)$ reversals to make any graph acyclic, we have $|E_{\mathrm{rev}}| \le \tfrac12 N(N{-}1)$. 
(ii) When $C$ fully specifies the ground truth, those relations are fixed and cannot be changed; in this case, we count the number of model predictions that contradict $C$, which gives the error rate, so it becomes plain error rate. 
Figure~\ref{fig:score} presents the calculation of the self-consistency score for different cases of $C$, as an example for the following equation. 

\begin{figure}
    \centering
    \includegraphics[width=\linewidth]{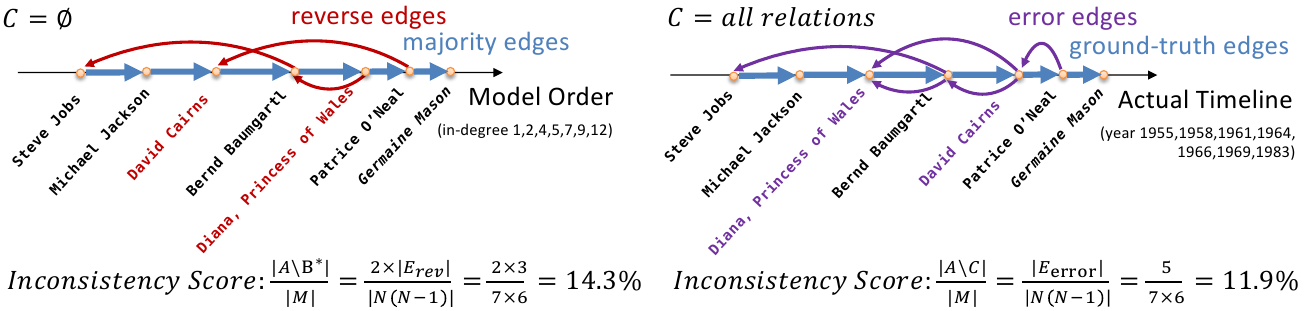}
    \caption{An example for calculating the inconsistency score for LLM's responses regarding people's birth dates comparison. When no context or ground-truth is given, we rank the objects by their in-degree, and the inconsistency score is based on the proportion of the reverse edges. If ground-truth is given, the exact order can be different from the majority order specified by the model's outputs (e.g. ``David'' and ``Diana'' should exchange positions), the inconsistency score becomes the proportion of error edges. Note that there can be multiple edges (e.g. ``$A$ is before $B$'', ``$B$ is after $A$'' produce two edges $(A,B)$.}
    \label{fig:score}
\end{figure}

\[
I(A;C) =
\left\{
\begin{array}{ll}
\displaystyle \frac{|A \setminus B^\star|}{|\mathcal{M}|} = \frac{|E_{\mathrm{rev}}|}{N(N-1)/2} & \text{if } C = \varnothing \quad \text{(edit-to-consistency)} \\[10pt]
\displaystyle \frac{|A \setminus \bar{C}|}{|\mathcal{M}|} = \frac{|E_{\mathrm{error}}|}{N(N-1)} & \text{if } C \text{ uniquely determines } \bar{C} \quad \text{(error rate)}
\end{array}
\right.
\tag{**}
\]

\paragraph{Key takeaways}
Equation (*) offers a single, task‑agnostic metric: \textbf{the smaller the score, the closer the LLM is to a self-consistent category.}
Equation (**) gives two specific cases to calculate the score that are used throughout our following experiments.

\section{Self-Consistency Evaluation for trending LLMs}
\label{sec:eval}
\subsection{Experimental Settings}
\paragraph{Datasets} We reuse the three 20-object temporal sets (\textsc{Art},
\textsc{Ancient}, \textsc{Recent}) and two U.\,S. geography sets
(\textsc{US-City} with 20 cities, \textsc{US-State} with 51 states)
from \citet{llm_space_time} as real-world data. We also construct some synthetic data, including a 20‑point plane, and an 11‑person four‑generation kinship tree. Detailed data statistics and examples can be found in Appendix~\ref{app:data}. 

\paragraph{Models} We adopt 12 models: GPT-o4-mini\citep{gpt-o4-mini}, GPT-4o\citep{gpt-4o}, GPT-o1-mini\citep{gpto1}, Qwen3-235B-A22B, Qwen3-30B-A3B\citep{yang2025qwen3}, DeepSeek R1, DeepSeek-R1-Distill-Qwen-7B, DeepSeek-R1-Distill-Llama-8B \citep{deepseekai2025deepseekr1incentivizingreasoningcapability}, DeepSeek V3 \citep{dsv3}, DeepSeek-V2-Lite-Chat \citet{dsv2}, Llama-3-8B-Instruct \citep{llama3modelcard} and Qwen2-7B-Chat\citep{qwen2}. We test models within 16B with their employment on an A800 machine to run less than 5 hours, using the Llama-Factory framework \citep{zheng2024llamafactory} with default hyperparameter configurations. We test the larger-scale models by calling APIs. The implementation details, including prompts and some case studies, can be found in Appendix~\ref{app:implementation}.

\subsection{Experimental Results: LLMs are still not self-consistent for simple 1D/2D/multi-hop reasoning}

This task aims to explore the following questions: Are current LLMs self-consistent? If not, to what extent is it not self-consistent? Do they exhibit different consistency on different structures?

We have some key findings through our experiments:
\begin{itemize}
    \item No models are completely self-consistent on simple 1D, 2D, and multi-hop reasoning tasks, when handling binary relations for 11 to 51 objects.
    \item All smaller models exhibit high inconsistency, including DeepSeek-R1-Distill-Qwen-7B and DeepSeek-R1-Distill-Llama-8B. Though the larger reasoning models like DeepSeek-R1 and GPT-o4-mini perform better, they're still hard to reach complete self-consistency.
    \item Graph method and EBM method calculate consistency scores from two directions respectively, and the final results have a positive correlation, which can confirm each other to effectively evaluate the degree of self-consistency of model responses. 
    \item When the error is sparse, we can use the graph method and EBM method to effectively fix the inconsistency and recover the ground-truth.
\end{itemize}

The detailed experimental results analyses are provided below.

\subsection{Self-consistency without context given}
The results are displayed in Table~\ref{tab:real}. Models can be divided into two groups: Larger models GPT-4o, o1-mini, DeepSeek-V3, and DeepSeek-R1 can exhibit $<20\%$ inconsistency score on most datasets; other smaller models and Qwen3-235B-A22B have at least $17.9\%$ inconsistency score and can reach up to $89.70\%$. 

\textbf{Temporal (1D)} For open-source models of different sizes, there's no clear relationship between the model size and its inconsistencies. For example, Llama3-8B outperforms the larger DeepSeek-V2-16B and the smaller R1-distill-Qwen-7B on all 1D temporal datasets.

\textbf{Spacial (2D)} Interestingly, the model's performance varied between east-west comparison and north-south comparison. This is because many LLMs make mistakes on the east-west comparison on minus longitudes (e.g., comparing $-82^\circ E$ and  $-91^\circ E$).

\begin{table}[ht]
  \centering
  \small
  \begin{adjustbox}{max width=0.7\textwidth}
  \begin{tabular}{lcccccc}
    \toprule
    \multirow{2}{*}{\textbf{Model}} &
      \multicolumn{3}{c}{\textbf{Temporal 1-D} $\downarrow$} &
      \multicolumn{2}{c}{\textbf{US Geo 2-D} $\downarrow$} \\
    \cmidrule(lr){2-4}\cmidrule(lr){5-6}
      & Art & Ancient & Recent
      & City (x/y) & State (x/y) \\
    \midrule
    GPT-4o          &  4.21 & 13.68 &  1.07 & \textbf{0.0/0.0} & 0.9/1.9 \\
    o1-mini         & 21.93 & 36.84 & 14.48 &  4.7/3.7 & 2.3/2.7 \\
    DeepSeek-V3     &  8.42 & 18.47 &  4.22 & \textbf{0.0}/1.1 & \textbf{0.4/0.9} \\
    DeepSeek-R1     & \textbf{3.68} & \textbf{9.50} & \textbf{1.05} & \textbf{0.0}/1.1 & \textbf{0.4/0.9} \\
    GPT-o4-mini     & 14.74 & 14.74 & 3.68 & 0.53/1.05 & 0.7/1.8 \\
    \midrule
    DeepSeek-V2-16B & 56.38 & 83.73 & 64.36 & 82.4/59.4 & 89.7/68.5 \\
    Llama-3-8B      & 23.94 & \textbf{25.93} & \textbf{27.03} & 82.2/25.5 & 88.2/24.0 \\
    R1-Dist-Llama-8B& 55.79 & 55.79 & 32.11 & \textbf{37.7}/23.2 & \textbf{35.4}/28.7 \\
    Qwen2-7B        & 23.22 & 56.08 & 50.13 & 83.9/\textbf{17.9} & 73.8/\textbf{20.7} \\
    R1-Dist-Qwen-7B & 63.68 & 61.05 & 49.60 & 62.6/43.5 & 58.4/35.7 \\
    Qwen3-30B-A3B   & 60.00 & 67.37 & 64.21 & 62.63/63.68 & 68.6/70.6 \\
    Qwen3-235B-A22B & \textbf{21.16} & 57.14 & 62.96 & 74.8/75.1 & 79.5/79.0 \\
    \bottomrule
  \end{tabular}
  \end{adjustbox}
  \caption{Inconsistency (\%) on all \emph{real-world} sets (lower is better).  
  For US City/State the two numbers are east-west / south-north inconsistency.}
  \label{tab:real}
\end{table}

\subsection{Self-consistency with ground-truth given}

\paragraph{2D fictional geographical relations} We have explored LLM's performance under three different ways of providing groud-truth context.
\begin{itemize}
    \item Provide XY Positions (XY Pos.): Give the (x, y) coordinates for all objects.
    \item Provide Relationships to a Central Element (Center-Rel.): One central element and all other quantitative relationships. ($\text{Object}_1$ to $\text{Object}_0$ is east (5 units) and north (2 units), $\text{Object}_2$ to $\text{Object}_0$ is ...).
    \item Provide Relationships by Ordered Pairs (Ordered-Rel.): Provide the successive quantitative relationships from $0$ to $n-1$ ($\text{Object}_0$ to $\text{Object}_1$, $\text{Object}_1$ to $\text{Object}_2$,...).
\end{itemize}

\paragraph{Kinship (Multi-hop)}

The task is to infer all gendered kinship relations in a four‐generation family tree from only 10 seed relations among 11 individuals (e.g., “A is the father of B”, “D is the wife of C”).  The 4‐generation tree has one couple in the highest generation, two couples in the second, multiple siblings in the third, and one daughter in the fourth.

\begin{table}[ht]
  \centering
  \small
  \begin{adjustbox}{max width=0.77\linewidth}
  \begin{tabular}{lcccc}
    \toprule
    \multirow{2}{*}{\textbf{Model}} &
      \multicolumn{3}{c}{\textbf{2-D Fictional Map} $\downarrow$} &
      \multirow{2}{*}{\textbf{Kinship} $\downarrow$} \\[-1pt]
    \cmidrule(lr){2-4}
      & XY Pos.\,(x/y) & Center-Rel.\,(x/y) & Ordered-Rel.\,(x/y) & \\
    \midrule
    GPT-4o          & \textbf{0.00}/\textbf{0.00} &  4.74/3.68 & 72.73/72.19 & 13.73 \\
    o1-mini         &  0.53/0.53                &  0.53/\textbf{0.00} & \textbf{6.32}/\textbf{6.32} & 8.82 \\
    DeepSeek-R1     & \textbf{0.00}/\textbf{0.00} & \textbf{0.00}/\textbf{0.00} & 26.70/26.16 & 9.80 \\
    DeepSeek-V3     & 51.05/58.95 & 31.05/43.68 & 81.28/73.94 & 10.78 \\
    GPT-o4-mini     & \textbf{0.00}/\textbf{0.00} & \textbf{0.00}/\textbf{0.00} & \textbf{0.00}/3.15 & \textbf{1.96} \\
    \midrule
    DeepSeek-V2-16B & 80.21/80.75 & 75.79/67.37 & 82.98/75.07 & 84.31 \\
    Llama-3-8B      & 59.47/30.08 & 61.74/43.97 & 76.19/76.88 & 68.63 \\
    R1-Dist-Llama-8B& 29.55/\textbf{22.11} & \textbf{17.37/10.58} & 79.79/72.78 & 33.33 \\
    Qwen2-7B        & 64.91/61.74 & 58.38/60.75 & 73.60/83.70 & 71.57 \\
    R1-Dist-Qwen-7B & \textbf{27.37}/30.77 & 23.16/14.78 & \textbf{55.14/57.30} & 57.84 \\
    Qwen3-30B-A3B   & 67.89/61.58 & 59.47/68.95 & 74.21/70.53 & 13.73 \\
    Qwen3-235B-A22B & 62.60/72.68 & 61.05/59.47 & 70.37/80.74 & \textbf{8.82} \\
    \bottomrule
  \end{tabular}
  \end{adjustbox}
  \caption{Inconsistency (\%) on \emph{synthetic} datasets (lower is better).}
  \label{tab:synthetic}
\end{table}

As Table~\ref{tab:synthetic} reports, explicit $(x,y)$ coordinates are easier for
GPT-4o and DeepSeek-R1; ordered-relation prompts substantially raise
error rates, confirming the difficulty of long-chain composition. Interestingly, GPT-4o's performance on the ordered relations task is comparable to that of smaller, open-source models. In contrast, GPT-o1 and DeepSeek-R1 perform better in such tasks involving long-chain reasoning. 


For kinship, even the longest reasoning path has length 5, all models exhibit non‐trivial inconsistency rates. Most errors stem from in‐law relationships and gender misidentification (e.g., DeepSeek‐R1 infers that E is both the great‐grandfather of Z and the grandmother of X).

\subsection{Two Methods for Self-Consistency Fixing \& Fixing Results}

As Section~\ref{sec:metric} states, calculating ``the proportion of edges that must be removed to eliminate all cycles'' is NP-complete, therefore the inconsistency score we give is actually an upper bound. Does it faithfully capture the actual inconsistency? Intuitively, we can think of the problem in another way: we can initialize a self-consistent system, and make it approximate the relations given by all models' responses. This innovate us of using energy-based models (EBM) optimization, illustrated in Figure~\ref{fig:method-overview}. We provide the fixing results of the graph method first, then the detailed procedure and fixing results of EBM. 

\begin{figure}[ht]
    \centering
    \includegraphics[width=0.6\linewidth]{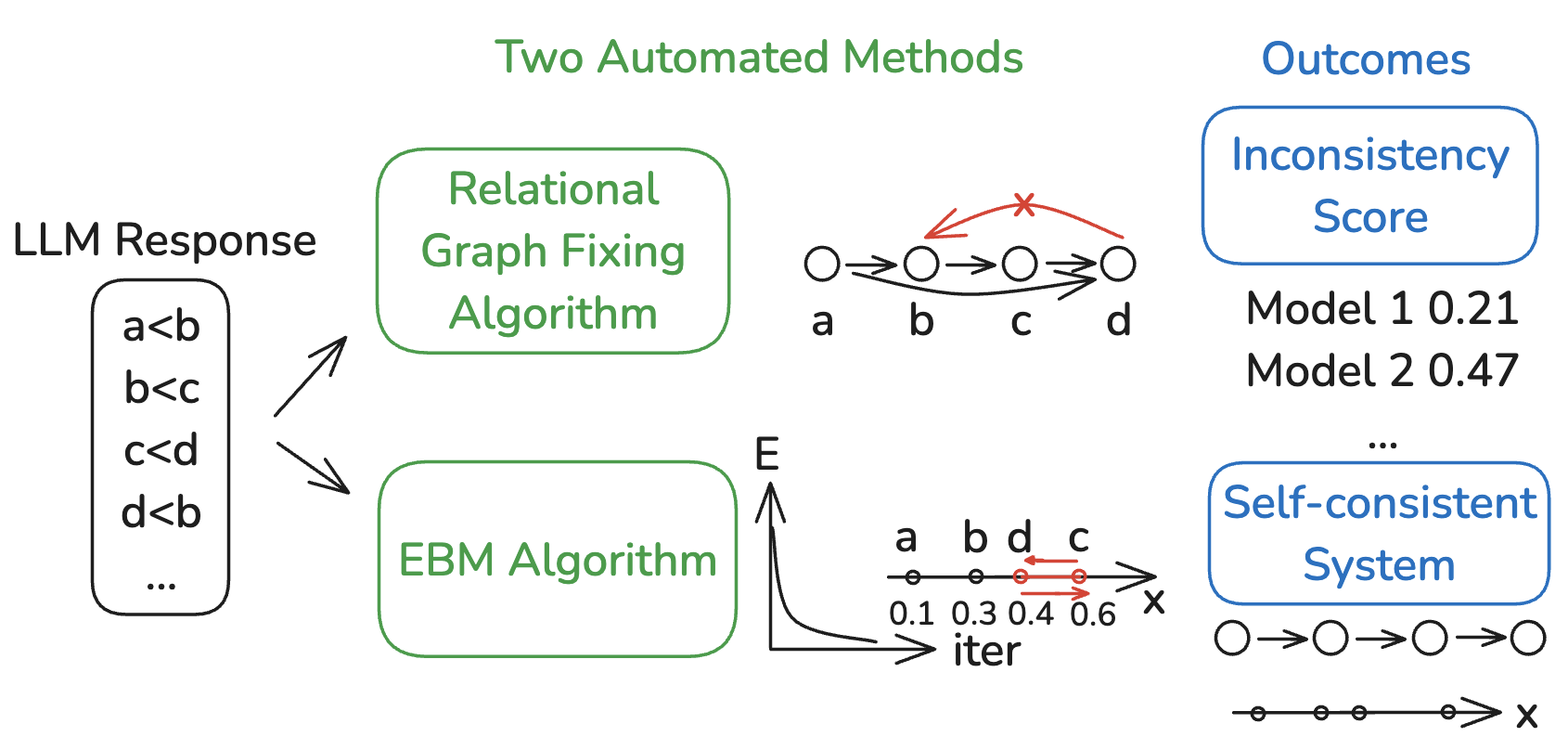}
    \caption{Two methods for self-consistency fixing: Graph Fixing Algorithm and EBM Algorithm. They approximate the self-consistent goal from opposite directions. The graph-based method starts from the inconsistent model’s raw relational predictions, dropping edges that cause the inconsistency, and finally reaches complete self-consistency. EBM takes an initial consensus assignment, preserving self-consistency all the way, and iteratively approximating to representing the model’s outputs by minimizing a well-defined energy (loss) function.}
    \label{fig:method-overview}
\end{figure}




\subsubsection{Graph Fixing Algorithm}
For effectively finding out reverse edges, we construct a directed graph per axis, and identify reverse edges via Tarjan SCCs and topological sort, then flip the minimal set to obtain a simply-ordered graph (detailed algorithm and proof in Appendix \ref{app:graph}).

\subsubsection{Graph Fixng Results}

After applying the above graph fixing method, we can get the total order of all places in both directions. We use the total order to reconstruct the self-consistent map and compare it with the reference map (reconstructed by the ground-truth total order of longitude and latitude). 

\begin{figure*}[ht]
    \centering
    \begin{subfigure}[b]{0.3\textwidth}
        \includegraphics[width=\textwidth]{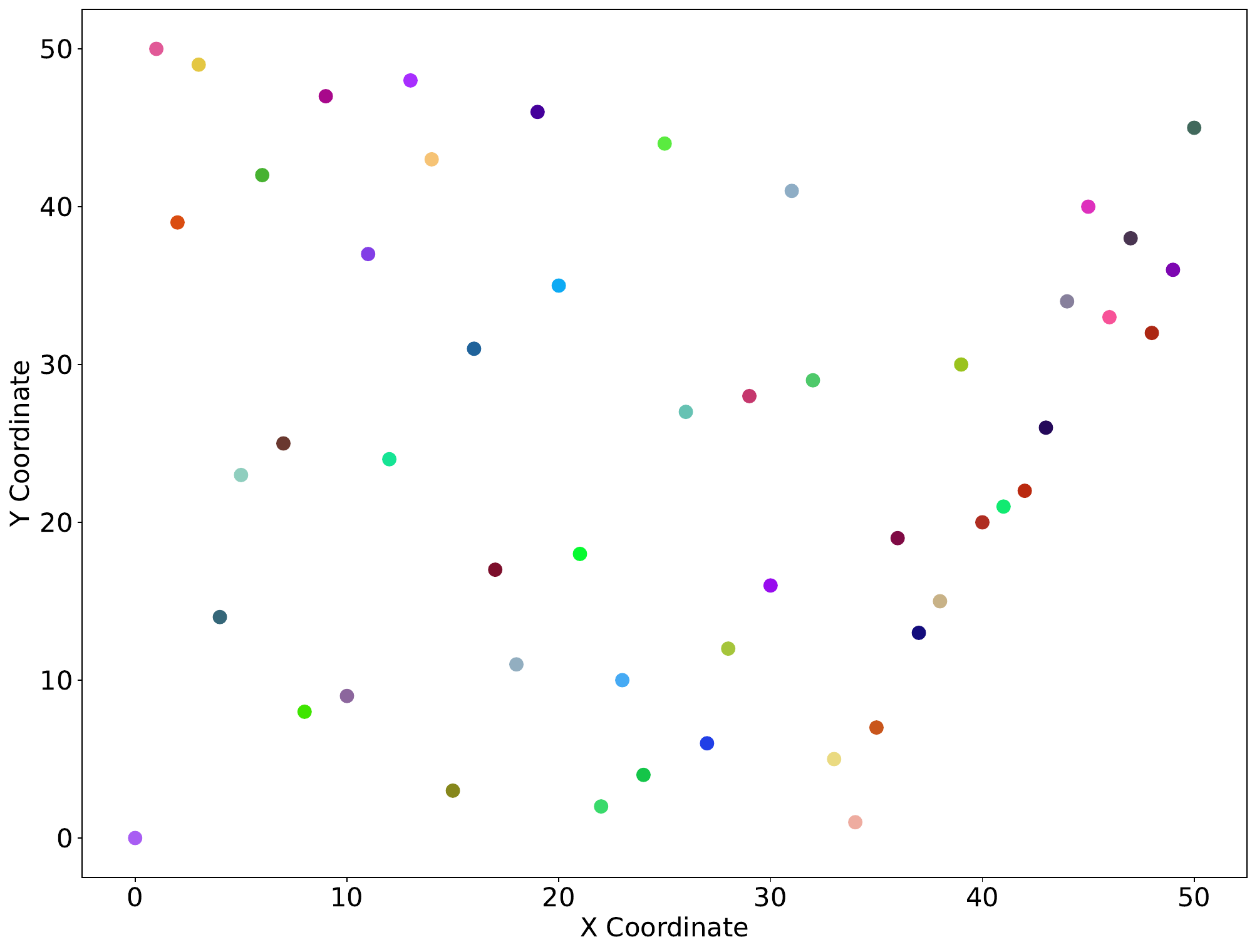}
        \caption{GPT-4o}
        \label{fig:sub2}
    \end{subfigure}
    \begin{subfigure}[b]{0.3\textwidth}
        \includegraphics[width=\textwidth]{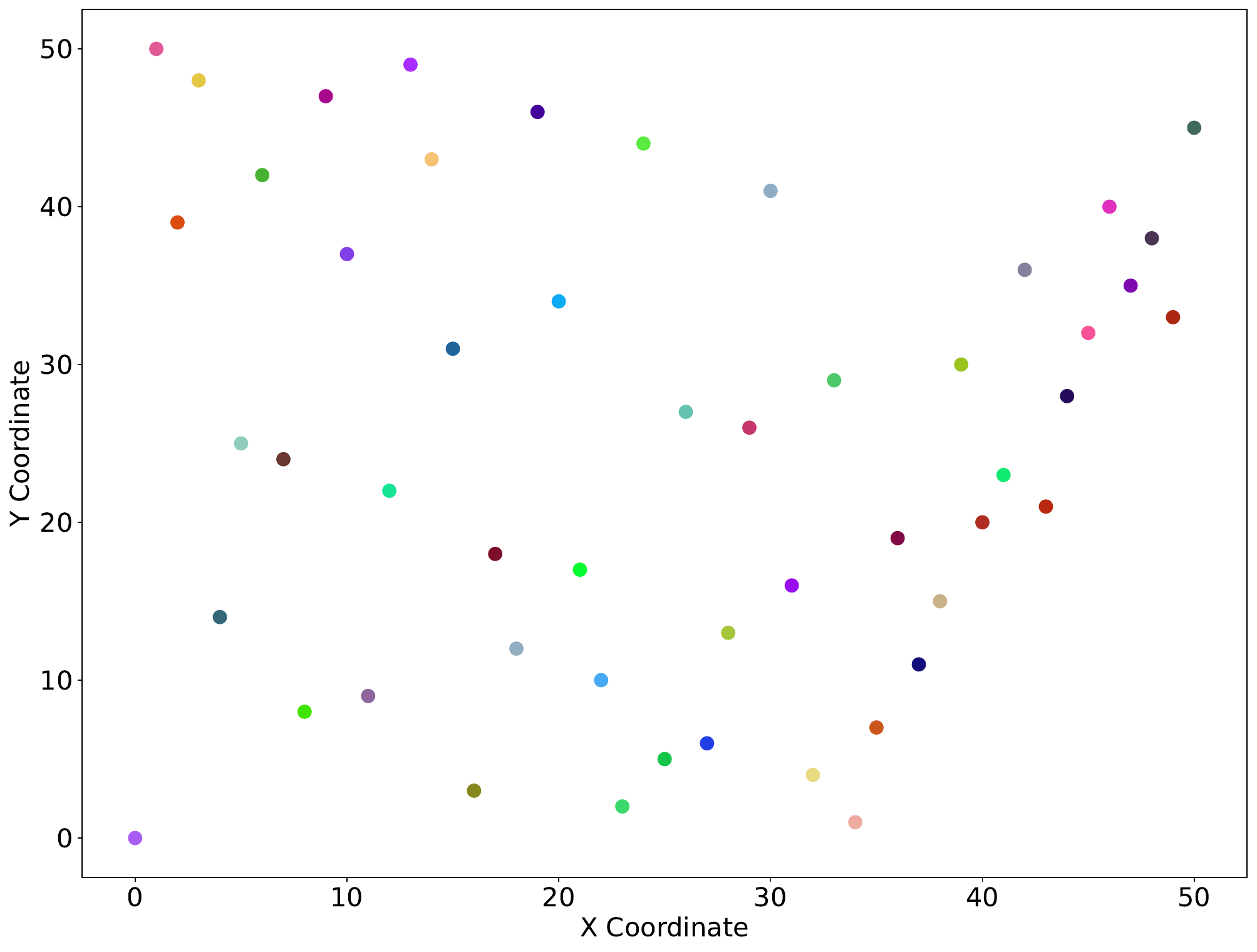}
        \caption{Ground truth}
        \label{fig:sub3}
    \end{subfigure}
    \begin{subfigure}[b]{0.3\textwidth}
        \includegraphics[width=\textwidth]{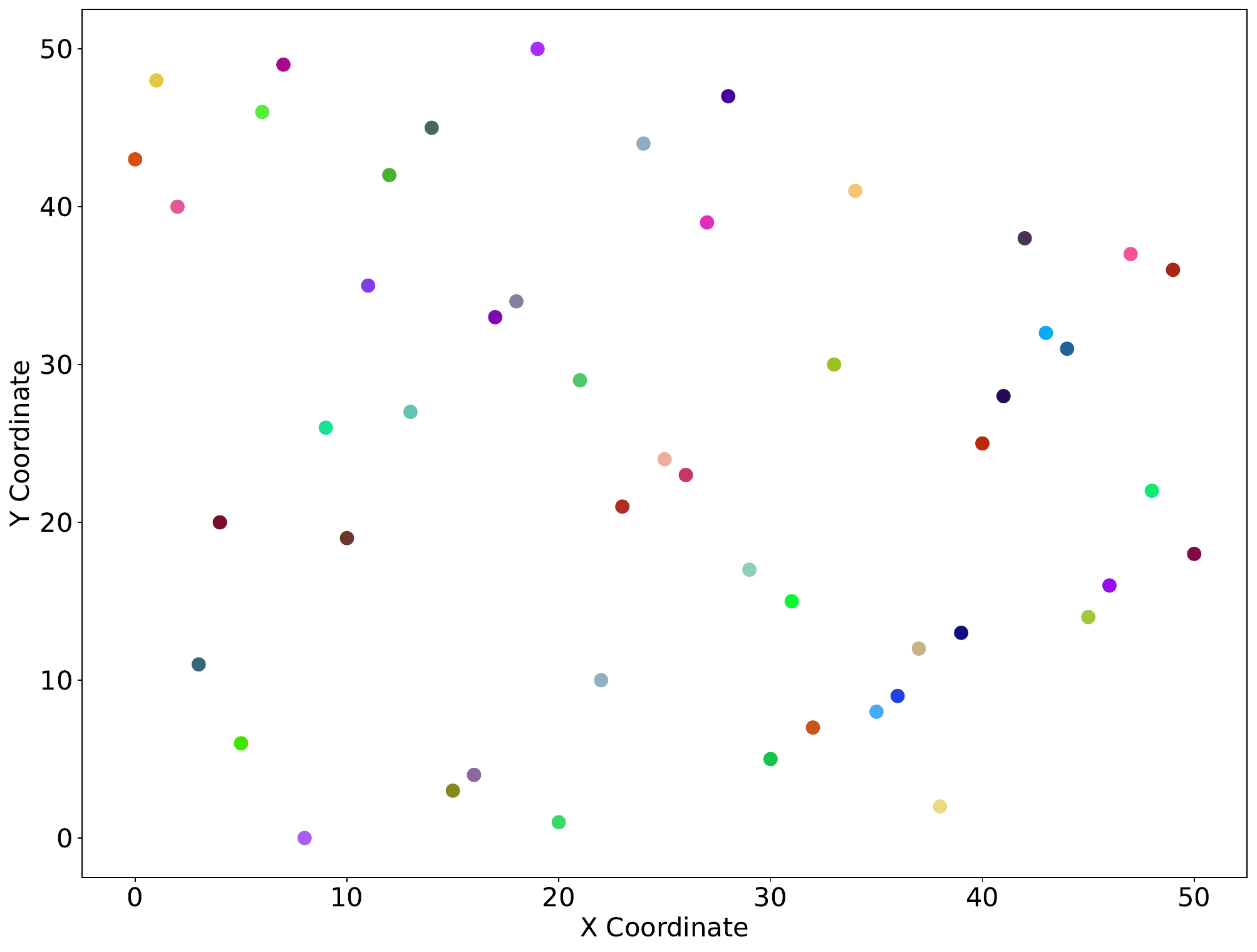}
        \caption{Qwen2-7B}
        \label{fig:sub4}
    \end{subfigure}
    \caption{The reconstructed graph of the US map of GPT-4o, the ground truth, and the Qwen2-7B-Instruct. From the point distribution, we can see that the GPT-4o is close to the ground truth, while Qwen2-7B-Instruct is far from the ground truth.}
    \label{fig:consistency blueprint}
\end{figure*}
We visualize some results for the states of the US in Figure~\ref{fig:consistency blueprint}. From the reconstructed map, we've found that the models with lower inconsistency score (e.g. GPT-4o) show overall distributions resembling the real map, although local positions still exhibit misalignments. In contrast, the models with higher inconsistency score (e.g. Qwen2-7B) demonstrate larger gaps in their overall distributions compared to the real map, with positional relationships approaching a random distribution. More visualization results for remaining models can be found in the Appendix\ref{app:2d-fix-result}.

\subsection{EBM Method}
\paragraph{Energy-Based Formulation}
The EBM method is suitable for any vector space. For simplicity, we talk about 1D case, the higher-dimensional case is its overlay in each dimension.
The methodology centers on an EBM framework where we:
\begin{enumerate}
    \item Represent each object as a scalar coordinate~$x_i$
    \item Define an energy function quantifying relation violations
    \item Optimize coordinates through gradient descent minimization
\end{enumerate}

\paragraph{Energy Function Specification}
For each directed relation $r_{ij}$ stating that "$x_i$ should precede $x_j$", we define the energy of the component $E(r_{ij})$ as: $E(r_{ij}) = \max(0,1 + (x_i - x_j))$\\
The total system energy aggregates individual relation energies:
\begin{equation}
    E_{\text{total}} = \sum_{r_{ij} \in \mathcal{R}} E(r_{ij})
\end{equation}
This formulation penalizes both incorrect orderings and insufficient separation in correct relations.

\paragraph{Optimization Protocol}
(1) Initialization: Sequentially initialize objects in [0,1]. First, assign random position to new object $x_k$. Then, For some established relations between $x_k$ and existing coordinates: If true relation direction with $x_j$ conflicts with current positions: change $x_k$ to $1.5x_j-0.5x_k$. (2) Gradient Descent: Update positions via:$x_i^{(t+1)} = x_i^{(t)} - \eta \frac{\partial E_{\text{total}}}{\partial x_i}$, 
    where $\eta$ is the learning rate. This gradually reduces energy by moving coordinates along the negative energy gradient.

\paragraph{EBM Fixing Results}

This formulation enables efficient identification and suppression of inconsistent relations while maintaining computational tractability through coordinate-based parameterization. As Figure~\ref{fig:energy-iterations} shows, the energy objective has fast and stable convergence. The method proves particularly effective under the sparsity-of-errors assumption, when most LLM-derived relations are fundamentally correct, as in Figure~\ref{fig:score-in-reverse}.

\begin{figure}[ht]
    \centering
        \begin{subfigure}[b]{0.52\textwidth}
            \includegraphics[width=\textwidth]{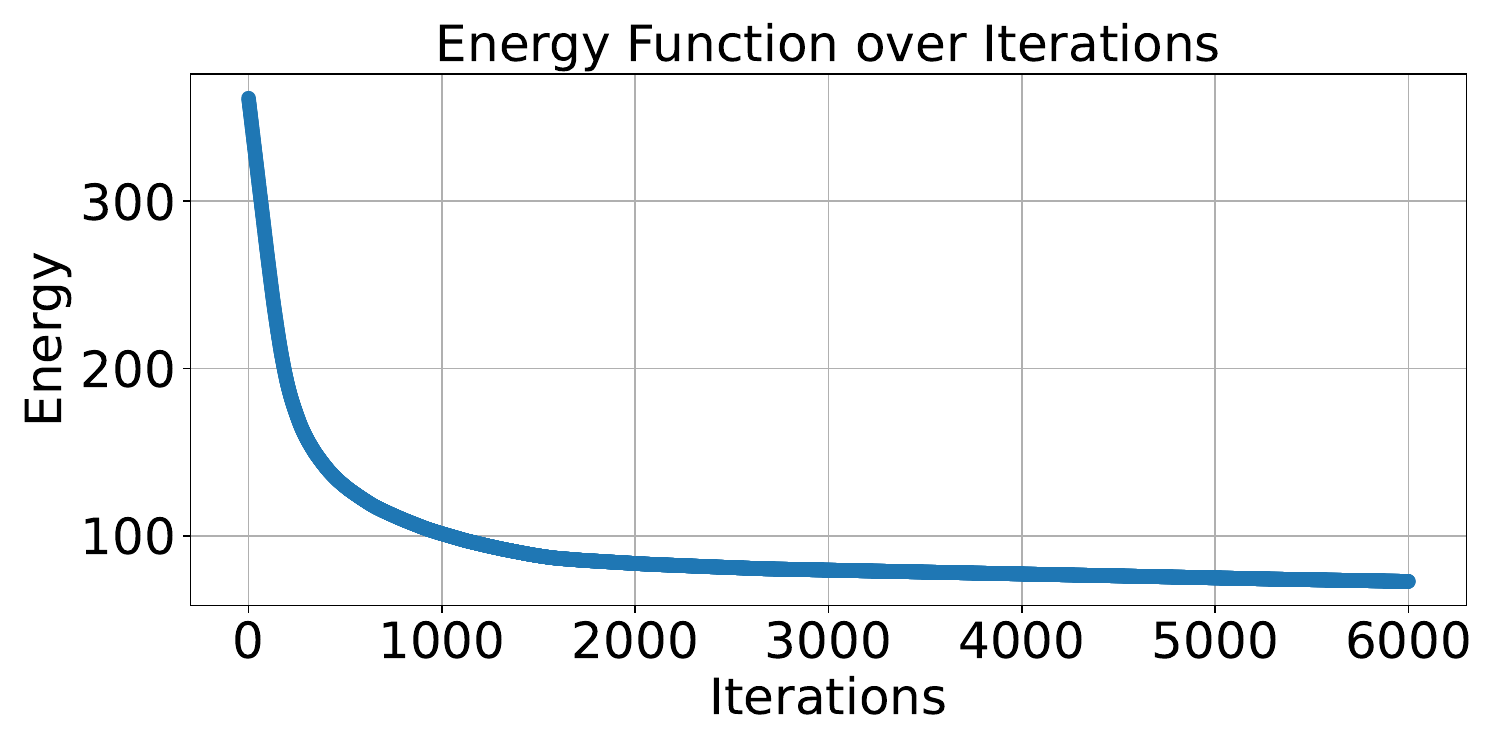} 
            \caption{The energy function over iterations of \\ DeepSeek-R1 on Ancient dataset}
            \label{fig:energy-iterations}
    \end{subfigure}
    \begin{subfigure}[b]{0.42\textwidth}
        \includegraphics[width=\textwidth]{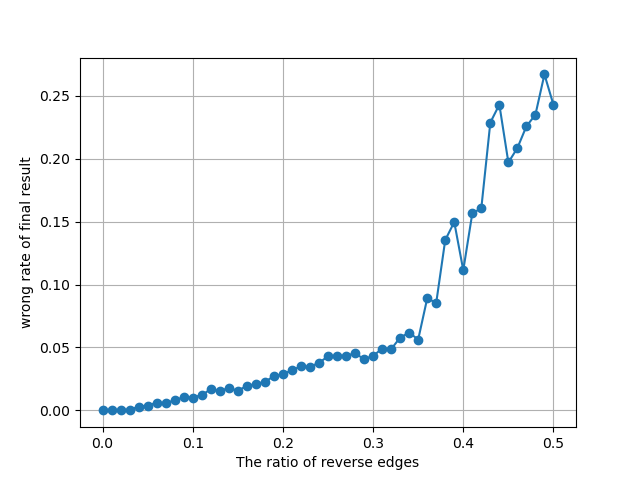}
        \caption{The error rate of EBM final result vs. the ratio of reverse edges}
        \label{fig:score-in-reverse}
\end{subfigure}
\caption{(a) EBM can effectively approximate the model's outputs within a few iterations. (b) We randomly reverse some edges from the true relation of dataset US\_state and apply EBM algorithm on it. We can almost find the true relation when the reverse edges are few. It's almost linear when ratio$<0.3$.}
\end{figure}

\subsection{Validation: the inconsistency scores have a positive correlation for the graph and EBM method}

After we apply the EBM Algorithm on the models' output, the final relations given by EBM is self-consistent. We can use it as a reference to calculate the inconsistency score of the original model's responses. The results are shown in Table~\ref{tab:EBM_fixed}.

\begin{table}[ht]
	\centering
	\small
	\begin{adjustbox}{max width=\linewidth}
		\begin{tabular}{lcccccccc}
			\toprule
			\multirow{2}{*}{\textbf{Model}} &
			\multicolumn{3}{c}{\textbf{Temporal 1-D} $\downarrow$} &
			\multicolumn{2}{c}{\textbf{US Geo 2-D (EBM)} $\downarrow$} &
			\multicolumn{3}{c}{\textbf{2-D Fictional Map (EBM)} $\downarrow$} \\[-1pt]
			\cmidrule(lr){2-4}\cmidrule(lr){5-6}\cmidrule(lr){7-9}
			& Art & Ancient & Recent
			& City (x/y) & State (x/y)
			& XY Pos.\,(x/y) & Center-Rel.\,(x/y) & Ordered-Rel.\,(x/y) \\
			\midrule
GPT-4o&5.79&14.21&4.74&\textbf{0./0.}&2.59/3.53&\textbf{0./0.}&2.63/3.68&75.79/74.21\\
o1-mini&18.95&38.42&20.53&2.63/2.63&2.27/2.67&0.53/0.53&0.53/0.&\textbf{5.26/6.32}\\
DeepSeek-V3&8.42&17.89&4.74&0./1.05&1.88/2.27&10.08/11.76&0.90/0.&54.79/49.32\\
DeepSeek-R1&\textbf{3.68}&\textbf{10.53}&\textbf{1.05}&0./1.05&\textbf{0.55/1.02}&\textbf{0./0.}&\textbf{0./0.}&28.95/29.47\\
\midrule
DeepSeek-V2-16B&75.26&102.11&112.11&102.63/79.47&99.06/81.88&83.16/79.47&74.74/76.84&85.79/73.16\\
Llama-3-8B&45.26&26.32&33.68&85.26/26.32&89.57/25.41&57.37/28.42&64.21/45.26&82.11/84.74\\
R1-Dist-Llama-8B&60.00&54.21&32.63&36.84/24.21&38.90/31.69&26.84/21.58&15.26/8.95&81.05/75.26\\	
Qwen2-7B-Chat&23.68&57.89&48.42&81.58/17.37&73.41/15.84&62.11/60.00&64.74/67.89&80.53/85.79\\
R1-Dist-Qwen-7B&66.32&67.37&57.37&63.16/42.63&58.51/35.92&27.89/30.00&21.58/15.26&58.42/62.63\\
\bottomrule
		\end{tabular}
	\end{adjustbox}
	\caption{Inconsistency scores computed with the EBM algorithm.  Boldface indicates the lowest (best) value in each column.}
	\label{tab:EBM_fixed}
\end{table}

The inconsistency rates produced by the Graph and EBM methods are extremely consistent across models, with every comparable column yielding a Pearson correlation of $r\geq 0.93$ and an overall correlation of $r=0.982$.
This confirms the validity of the two models' evaluation and fixing results.

\section{Related Work}

\paragraph{Self-Consistency of LLMs}

Previous work on self-consistency mainly focuses on the failure of multi-step reasoning for math problems. \citet{huang24self-correct} demonstrates that LLMs fail to self-correct intrinsic reasoning errors without external feedback, while \citet{ChenPPPZBC24two-failure-self-consistency} further classifies self-consistency issues in math reasoning into two types, hypothetical and compositional consistency. \citet{DBLP:chen23universal-self-consistency} findsthat  LLMs have universal consistency to do sampling and aggregation for correct answers. Their work provides valuable phenomenal insights and fixes inconsistencies by aggregation. Our proposed approach differs from the above works in that we provide a theoretical framework and quantitative evaluation metrics to further investigate self-consistency issues quantitatively, and propose automated algorithms for fixing self-consistency.

\citet{inconsistency-abound} shows that LLMs exhibit inconsistencies when reasoning concepts within a knowledge graph, while our work studies inconsistencies in reasoning in low-dimensional space. \citet{zhang2023trade} introduces the Impossible Trinity Theorem, demonstrating that for removal-based interpretability algorithms, consistency and efficiency cannot be simultaneously achieved, 
which highlights fundamental limitations in removal-based algorithms.

\paragraph{LLM Interpretability}
\citet{tan2023contrastive} analyzes the internal representations of SimCLR and CLIP and prove that these models maintain a consistent internal space, where the similarity between objects can be directly computed based on their relative positions in this learned space. Their findings suggest that some models inherently develop structured representations that preserve consistency. 
\citet{yuan2023categorical} propose a categorical perspective on interpretability, framing explanations as the application of external knowledge, such as limit decompositions of input and output, to ensure interpretability is both consistent and verifiable.

\section{Conclusion}
\label{sec:conclusion}
The challenges of self-consistency and interpretability in large models are complex, yet addressing them is critical for the responsible deployment of AI systems, particularly in high-stakes fields. 

We propose a new evaluation metric for LLM self-consistency, and test competitive LLMs of different sizes. Our preliminary findings suggest that while simple 1D temporal, 2D spatial, and multi-hop kinship relationships may be trivial for human reasoning, even state-of-the-art models including GPT-o4-mini and DeepSeek-R1 struggle to handle these situations.

Our work also provides two algorithms for finding the relations that are contradict to the LLM's majority of answers automatically. How to make LLMs trained to be inherently self-consistent? We call for more attention to future research in the field of self-consistency.

\bibliographystyle{plainnat}
\bibliography{reference}

\appendix
\newpage

\section{Appendix}

\subsection{Category Theory View}
\label{app:ct}

Traditional frameworks typically treat LLMs as individual, isolated input-output mappings: evaluation is solely based on ($input_i$, $output_i$, $golden_i$) triplets, neglecting relational structures across outputs. In contrast, our category-theoretic framework explicitly formalizes relational consistency across multiple outputs ($output_i$, $output_j$). It allows us to systematically answer crucial questions: Are these outputs internally coherent, or do they contradict each other?
Here we outline some necessary definitions used in our paper.

\begin{definition}
\label{def1}
(Category, Object, Morphism). A category $\mathcal{C}$ has a set of objects $\mathrm{Ob}(\mathcal{C})$ and a set of morphisms $\mathrm{Hom}_\mathcal{C}(a,b)$, from $a$ to $b$ for every $a,b\in \mathrm{Ob}(\mathcal{C})$. 
Given $f \in \mathrm{Hom}_{\mathcal{C}}(a, b), g \in \mathrm{Hom}_{\mathcal{C}}(b, c)$, there is a unique composition $g \circ f \in \mathrm{Hom}_{\mathcal{C}}(a, c)$. Two axioms govern the composition of morphisms:

- \emph{Identity}: For every object $x$, there exists a identity morphism $\mathrm{id}_x: x \rightarrow x$ for $x$, such that for every morphism $f: a \rightarrow b$, we have $\mathrm{id}_b \circ f=f=f \circ \mathrm{id}_a .$
   
- \emph{Associativity}: If $f: a \rightarrow b, g: b \rightarrow c$, and $h: c \rightarrow d$ then $h \circ(g \circ f)=(h \circ g) \circ f .$

\end{definition}

\begin{definition}
 (Functor). 
A functor $F$ from a category $\mathcal{C}$ to a category $D$, written as $F: \mathcal{C} \rightarrow D$, maps an object $x \in \mathrm{Ob}(\mathcal{C})$ to $F(x)\in \mathrm{Ob}(D)$; as well as a morphism $f: x \rightarrow y$ in $\mathrm{Hom}(\mathcal{C})$ to $F(f): F(x) \rightarrow F(y)$ in $D$, such that the following two properties hold:
  
- \emph{Preservation of Identity}: For every object $x$ in $\mathcal{C}, F(\mathrm{id}_x)=\mathrm{id}_{F(x)}$;

- \emph{Preservation of Composition}: For all morphisms $f: x \rightarrow y$ and $g: y \rightarrow z, F(g \circ f)=F(g) \circ F(f)$.
\end{definition}

\begin{application}
Applying the above notion, we can have a unified problem formulation across different relational reasoning tasks.

Let us consider a set of objects $A, B, C \in \mathrm{Ob}(\mathcal{C})$ representing entities in the world (e.g., events, places or people), and morphisms $f: A \rightarrow B$ and $g: B \rightarrow C$ representing relational knowledge between them (e.g., "A is before B", "B is before C ").

A model is said to be ``relationally self-consistent'' if, for all such morphisms learned or generated by the model, there exists a morphism $h: A \rightarrow C$ such that $h = g \circ f$ and the interpretation of $h$ is semantically consistent with the composition of $f$ and $g$. In other words, if a model affirms $f$ and $g$, it should not contradict the composition $g \circ f$.
\end{application}

However, LLMs may produce responses that violate this composition property, i.e., they might affirm $f$ and $g$ but reject $g \circ f$. To address this, we introduce the notion of a \textbf{self-consistent proxy category} $\mathcal{C}_{proxy}$: an approximately minimal adjustment of the model's output that forms a valid category with associativity and identity preserved. 

\begin{definition}
\label{def:proxy}
(Self-consistent proxy category). Formally, for a given set of inferred morphisms $\mathcal{M}$ from an LLM, we define $\mathcal{C}_{proxy}$ as a category $(\mathrm{Ob}(\mathcal{C}), \mathrm{Hom}_{\mathcal{C}})$ such that:

- Each morphism in $\mathcal{M}$ is either preserved or replaced with a corrected morphism;

- The resulting morphisms obey associativity: $(h \circ g) \circ f = h \circ (g \circ f)$ for all applicable $f,g,h$;

- Identity morphisms $\mathrm{id}\_x$ are present for all $x \in \mathrm{Ob}(\mathcal{C})$.
\end{definition}

This proxy serves as a consistency-preserving abstraction of the model’s beliefs.


Suppose we have multiple self-consistent proxy categories $\mathcal{C}_{model}^{(i)}$, each representing the internal relational logic of a specific language model. Let $\mathcal{C}_{reality}$ denote an external, ground-truth relational structure (e.g., actual temporal or spatial orderings), we define the reality-alignment functor as follows.

\begin{definition}
(Reality-alignment functor). We define a reality-alignment functor $F: \mathcal{C}_{model}^{(i)} \rightarrow \mathcal{C}_{reality}$ that attempts to map objects and morphisms in the model’s category to their counterparts in the real world. This functor should satisfy:
\end{definition}

- \emph{Object Mapping}: $F(x)$ is the real-world interpretation of the model’s object $x$.

- \emph{Morphism Mapping}: For every morphism $f: x \rightarrow y$ in $\mathcal{C}_{model}^{(i)}$, $F(f): F(x) \rightarrow F(y)$ should reflect the ground-truth relationship.

We say a model is \textit{reality-aligned} if such a functor $F$ exists that also satisfies:

- Preservation of Identity: $F(\mathrm{id}_x) = \mathrm{id}_{F(x)}$;

- Preservation of Composition: $F(g \circ f) = F(g) \circ F(f)$ for all $f, g$.






\subsection{Graph Method For Inconsistency Evaluation and Fixing}
\label{app:graph}
Can we always transform the graph to a simply-ordered graph (defined in Definition~\ref{def:simply_ordered_graph}) by finding and deleting reverse edges? We prove Proposition~\ref{prop:algorithm_proof}, and then Theorem~\ref{thm:simply_ordered_graph}, which provides a theoretical guarantee for achieving self-consistency by removing reverse edges in the graph following Algorithm \ref{alg:node_ordering}. 

\begin{proposition}
	\label{prop:algorithm_proof}
	There exists an algorithm that, for any directed graph $G = (V,E)$, can assign an order $f: V \to \mathbb{R}$ to its nodes such that: If a cycle is formed during the ordering process, the cycle contains exactly one edge $(u,v)$ where $f(u) > f(v)$ (a reverse edge), and this reverse edge will be removed to eliminate the cycle
\end{proposition}

\begin{proof}
	We prove this by construction and induction. Consider Algorithm~\ref{alg:node_ordering}:
	
\begin{algorithm}[ht]
	\caption{Node Ordering Algorithm}
	\label{alg:node_ordering}
	\begin{algorithmic}[1]
		\STATE \textbf{Input:} A directed graph $G = (V, E)$
		\STATE \textbf{Output:} An ordering function $f: V \to \mathbb{R}$
		\STATE Decompose $G$ into strongly connected components (SCCs) $S_1, S_2, \dots, S_n$ using Tarjan's algorithm, ordered topologically
		\STATE Let $G'$ be the DAG of SCCs (condensation of $G$)
		\STATE Initialize counter $t \gets 1$
		\FOR{each SCC $S_k$ in topological order}
		\STATE Choose arbitrary root node $s \in S_k$
		\STATE Initialize visited set $N \gets \{s\}$
		\STATE Set $f(s) \gets t$, $t \gets t + 1$
		\WHILE{$N \neq S_k$}
		\STATE Select $v \in S_k \setminus N$ with $(u,v) \in E$ for some $u \in N$
		\STATE Remove all edges $(v,w)$ where $w \in N$ (these would form reverse edges)
		\STATE Set $f(v) \gets t$, $t \gets t + 1$
		\STATE Update $N \gets N \cup \{v\}$
		\ENDWHILE
		\ENDFOR
		\STATE \textbf{Return} $f$
	\end{algorithmic}
\end{algorithm}
	To prove correctness, consider any cycle $C$ that exists after ordering. By construction: All nodes in $C$ must belong to the same SCC $S_k$. Within each SCC:
	\begin{enumerate}
		\item Nodes are ordered sequentially, maintaining that all edges between ordered nodes go from lower to higher order
		\item When adding node $v$, any edge from $v$ to previously ordered nodes is removed
	\end{enumerate}
	Thus, any cycle must use exactly one reverse edge (from $v$ to an earlier node).
\end{proof}

\begin{proposition}
\label{prop:weakly_connected}
Algorithm~\ref{alg:node_ordering} preserves weak connectivity. That is, if $G$ is a weakly connected graph, the resulting graph after edge reduction remains weakly connected.
\end{proposition}
\begin{proof}
The DAG of SCCs is weakly connected, and the algorithm leaves it entirely unchanged. Within each SCC, every new node added must be connected to some previously processed node, maintaining weak connectivity. Therefore, any pair of nodes in the graph remains weakly connected.
\end{proof}

\begin{theorem}
	\label{thm:simply_ordered_graph}
	If we modify Algorithm~\ref{alg:node_ordering} reverse the edges instead of removing reverse edges, we can get a simply-ordered graph.
\end{theorem}

\begin{proof}
	After processing:
	\begin{itemize}
		\item All original edges between nodes satisfy $f(u) < f(v)$ (by construction)
		\item All reversed edges $(v,u)$ become $(u,v)$ with $f(u) < f(v)$
	\end{itemize}
	Thus all edges respect the order, satisfying Definition~\ref{def:simply_ordered_graph}.
\end{proof}

For implementation, since the root node in Step 7 can be chosen arbitrary, but we want to make the outcome more close to the model's outputs, we rank the nodes according to their degree (as a reference order) and remove edges contradict to this order. This gives the final graph fixing algorithm used in our experiments to give the inconsistency score and do the fixing.

\subsection{Implementation Details}
\label{app:implementation}

\subsubsection{Dataset Summary}
\label{app:data}
The dataset statistics are shown in Table~\ref{tab:dataset-summary}.

\begin{table}[h!]
\centering
\begin{tabular}{lll}
\toprule
Dataset & Size & Example Objects \\ 
\midrule
Art & 20 & ``Life of Pi (Yann Martel)'', ``King Arthur (Antoine Fuqua)''\\
Ancient Figure & 20 & ``Hugh Primas (Latin lyric poet)'', ``Wace (Norman writer from Jersey)''\\
Recent Figure & 20 & ``Michael Jackson'', ``Steve Jobs''\\
   \midrule
US\_City & 20 & ``New York (City)'', ``Los Angeles (City)'', ``Seattle (City)''\\
US\_State & 51 & ``California'', ``Washington State'' \\
Artificial Place & 20 & ``Object\_0'', ..., ``Object\_19'' \\
\midrule
Kinship & 11 & ``A'', ``B'', ``X'', ``Y'', ``Z'' ...\\
\bottomrule
\end{tabular}
\caption{\label{tab:dataset-summary}
Dataset Summary. The first four datasets are adopted from \citep{llm_space_time} with MIT License. We add a US\_State dataset for a larger number of objects to show the model's performance on harder cases, datasets for artificial place and kinship, where all LLMs won't differ in prior knowledge.}
\end{table}


\paragraph{Binary relation Types}
Each dataset in Table~\ref{tab:dataset-summary} is associated with a specific set of binary relations, as listed below:
\begin{itemize}
    \item \textbf{Art, Ancient Figure, Recent Figure (Temporal reasoning):} \texttt{before}, \texttt{after}
    \item \textbf{US\_City, US\_State, Artificial Place (Geospatial reasoning):} \texttt{northeast}, \texttt{southeast}, \texttt{southwest}, \texttt{northwest}
    \item \textbf{Kinship (Family reasoning):} \texttt{husband}, \texttt{wife}, \texttt{father}, \texttt{mother}, \texttt{son}, \texttt{daughter}, \texttt{brother}, \texttt{sister}, \texttt{uncle}, \texttt{aunt}, \texttt{cousin}, \texttt{niece}, \texttt{nephew}, \texttt{grandpa}, \texttt{grandma}, \texttt{grandson}, \texttt{granddaughter}, \texttt{brother-in-law}, \texttt{sister-in-law}, \texttt{father-in-law}, \texttt{mother-in-law}, \texttt{son-in-law}, \texttt{daughter-in-law}, \texttt{aunt-in-law}, \texttt{uncle-in-law}, \texttt{niece-in-law}, \texttt{nephew-in-law}, \texttt{great-grandpa}, \texttt{great-grandma}, \texttt{great-grandson}, \texttt{great-granddaughter}
\end{itemize}
These relation sets reflect the type of structured reasoning each dataset is intended to evaluate—temporal, spatial, or familial. Those relations are given models in the prompts, and ask the model to select one as final answer.

\textbf{Temporal Datasets (1D)} Our first three temporal datasets are a subset of the classical dataset proposed by \cite{llm_space_time}. It consists of (1) Art: the titles (and creators) of artworks released between 2000 to 2009, with Wikipedia page views $>1,000,000$; (2)Figure: the names (and occupations) of historical figures who were born between $1000$ AD and $2000$ AD adapted from \citep{AnnamoradnejadA22}. In the task, we ask the models to compare the release date and birth date.

\textbf{Spacial Datasets (2D)} We adopt two real-world datasets of place within the United States, for city and state names.
The models are asked to compare the geographical centroid longitude and latitude of those places to calculate the east-west and north-south relationships. 
We also construct an artificial dataset, sampling 20 points from a 2D plane with twenty coordinates $x_i, y_i$ are essentially random permutations of integers from -10 to 9. 

\textbf{Kinship Dataset} The task involves inferring family tree relationships over four generations using 10 given relationships for 11 people (e.g. A is the father of B, D is the wife of C) that can determine all the gender and relationships in the family tree. The tree has 4 generations, with the highest generation of the family tree has only one couple, the second generation has two couples, and the next generation has many brothers and sisters, one of whom has a daughter of the fourth generation. This is our synthetic dataset, and more cases can be automatically generated through our program.

\subsubsection{Prompts}

Our framework utilizes two primary types of prompts: history comparison prompts and geographical relationship prompts. Each prompt is carefully structured to elicit precise reasoning and standardized outputs from the language model.

\paragraph{1D History Prompts}

We implement three variants of history comparison tasks:

\begin{itemize}
    \item \textbf{Art Work Release Date Comparison}: Determines which artwork was released earlier
    \item \textbf{Historical Figures Birth Date Comparison}: Compares birth dates of historical figures
    \item \textbf{Fictional Event Occurrence Comparison}: Analyzes the chronological order of fictional events
\end{itemize}

The general structure for history comparison prompts follows:

{\footnotesize
\begin{verbatim}
# Task: [Task Title]
Analyze the dates of {objectA} and {objectB},
then decide whether {objectA} is before or after
{objectB}.

# Example Format:
## Step-by-step analysis:
... (detailed reasoning here) ...
## Answer:
{objectA} is \\boxed{...} {objectB}.

# Important Notes:
- The answer should end with the sentence: 
'{objectA} is \\boxed{{...}} {objectB}.'
- The value inside the \\\\boxed{{...}} 
should be one of the following: before, 
after.
- The reasoning should be clear and consistent, 
with no conflicting statements.
\end{verbatim}
}

\paragraph{2D Geometry Prompts}

We implement three variants of geographical relationship tasks:

\begin{itemize}
    \item \textbf{US Locations}: Analyzes relative positions of real locations using latitude and longitude
    \item \textbf{Fictional Places}: Determines relative positions in a 2D coordinate system
    \item \textbf{Inconsistent Information}: Evaluates geographical relationships when presented with potentially inconsistent data
\end{itemize}

The general structure for geographical relationship prompts follows:

{\footnotesize
\begin{verbatim}
# Task: [Task Title]
[Optional context]
Analyze the relative position of {objectA} with 
respect to {objectB}.

# Example Format:
## Step-by-step analysis:
... (detailed reasoning here) ...

## Answer:
{objectA} is \boxed{...} of {objectB}.

# Important Notes:
- Answer format: '{objectA} is \boxed{...} of 
{objectB}.'
- Use only: southeast, southwest, northeast, 
northwest.
- The reasoning should be clear and consistent, 
with no conflicting statements.
\end{verbatim}
}

For the inconsistency detection variant, the prompt includes additional instructions to output ``INCONSISTENT'' when contradictory information is detected.

All prompts are designed to enforce structured reasoning and standardized output format, facilitating automated evaluation and consistency analysis across different model responses.

\paragraph{Multi-hop Kinship Prompts}
{\footnotesize
\begin{verbatim}
# Task: # Task: Family Relationship Analysis

1. Z is the son of Y.
2. C is the son of E.
3. A is the daughter of F.
4. D is the father of B.
5. X is the daughter of A.
6. C is the father of Z.
7. E is the husband of F.
8. B is the son of A.
9. W is the daughter of D.
10. G is the daughter of X.

Analyze the family relationship between {objectA} and {objectB}.

# Example Format:
## Step-by-step analysis:
... (detailed reasoning here) ...

## Answer:
{objectA} is the \boxed{...} of {objectB}.

# Important Notes:
- The answer should end with the sentence: 
  `{objectA} is the \boxed{...} of {objectB}.'
- The value inside the \boxed{...} should be one 
  of the following:
  husband, wife, father, mother, son, daughter, 
  brother, sister, uncle, aunt, cousin, niece, 
  nephew, grandpa, grandma, grandson, 
  granddaughter, brother-in-law, sister-in-law, 
  father-in-law, mother-in-law, son-in-law, 
  daughter-in-law, aunt-in-law, uncle-in-law, 
  niece-in-law, nephew-in-law, great-grandpa, 
  great-grandma, great-grandson, 
  great-granddaughter.
- The reasoning should be clear and consistent, 
  with no conflicting statements.
\end{verbatim}
}

\subsubsection{Case Study in LLM's responses}

\begin{figure}[ht]
    \centering
    \includegraphics[width=\linewidth]{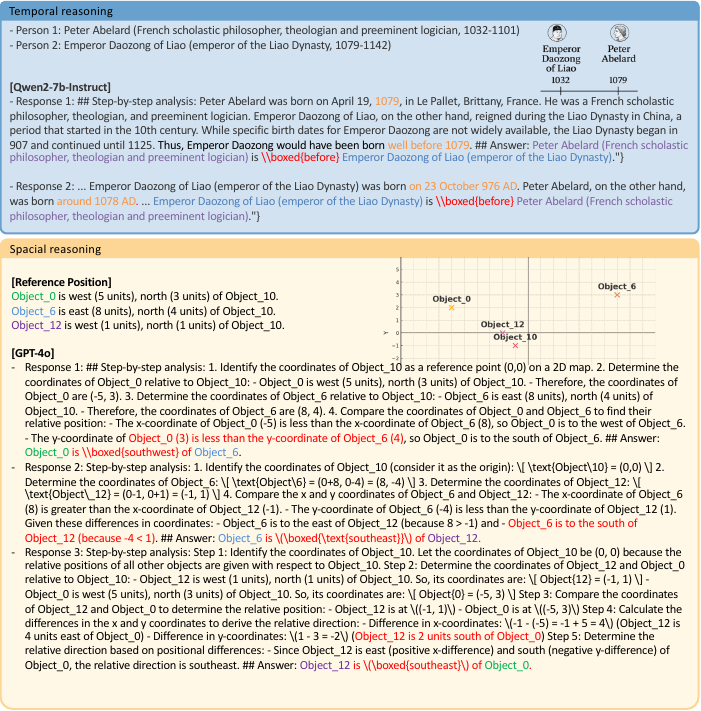}
    \caption{A case study of inconsistencies found.}
    \label{fig:failure_case}
\end{figure}
Figure~\ref{fig:failure_case} shows an example of models that provide inconsistent responses.
Through case studies on LLM's responses, we have found that there are some main reasons for inconsistency: 
\begin{itemize}
    \item Factual Inconsistencies: Models may think an event happens on different dates.
    \item Bias: models prefer a specific relationship over others. E.g. The number of answers ``before'' more than ``after''.
    \item Error within the reasoning process, especially when the model exchanges the analysis order of objects (e.g. Ask the relation of A to B, model first analyzes B then A, and in the final answer exchange it A to B but not flipping the relations.) or the analysis includes calculation of minus values.
\end{itemize}

\subsection{Limitations}
\label{app:limit}
Despite the new findings on LLM self-consistency, our work has several limitations that should be addressed in future research.

First, self-consistency remains a challenging and nascent problem. While our study focuses on relatively simple 1D and 2D binary relationships, which we demonstrate are often problematic even for state-of-the-art large language models (LLMs), self-consistency encompasses broader aspects. These include the internal consistency of mathematical proof derivations, reasoning within specific domains (e.g., law, medicine, finance), and natural language inference tasks (such as role-playing games like Werewolf). Our current work primarily introduces an evaluation metric and consistency-fixing methods for simpler tasks, and while these methods are extendable to common temporal and spatial relationships, a more comprehensive framework for assessing self-consistency in more complex domains remains an open challenge.

Second, our work highlights the importance of LLM interpretability, requiring both self-consistency and alignment with reality. However, our approach is mainly focused on achieving self-consistency, and we have not yet identified effective methods to simultaneously ensure both self-consistency and alignment with reality. We suspect this may represent an inherent difficulty, possibly necessitating new model architectures or specialized training tasks to address both aspects together.

\subsection{Other Interesting Experimental Results}
\subsubsection{Do LLMs aware of inconsistent relations?}

Under the ``2D fictional geographical task Ordered-Rel. setting'', we give $n$ relationships for $n$ objects that are contradictory, and ask the model to ``infer the relation between two objects'' or ``points out that the relationships are inconsistent''. We use 380 samples, and test to see if the model can recognize the contradiction.

\begin{table}[ht]
\centering
\begin{adjustbox}{max width=\linewidth}
\begin{tabular}{lccccccccccc}
\toprule
Model & GPT-4o & o1-mini & o4-mini & DS-V3 & DS-R1 & DS-V2-16B & DS-R1-Qwen & DS-R1-Llama & Llama-3-8B & Qwen2-7B \\
\midrule
Success (\%) & 15.00 & 91.32 & \textbf{98.95} & 38.16 & 58.42 & 12.37 & 41.32 & 31.05 & 9.47 & 21.84 \\
\bottomrule
\end{tabular}
\end{adjustbox}
\caption{\label{tab:success_rates_full}
Success rate (\%) of each model in detecting inconsistencies among contradictory 2D relational inputs.
}
\end{table}

For Table~\ref{tab:success_rates_full}, we have found that: models such as GPT-4o, DeepSeek-V2-16B, and Llama-3-8B-Instruct demonstrate very weak ability to recognize inconsistent relations, with success rates below 15\%. In contrast, reasoning models like o1-mini and o4-mini can flag contradictions. Notably, both instruction tuning and model size also affect this ability.

\subsubsection{2D Consistency Fixing Results}
\label{app:2d-fix-result}
Do fixing results of non-fictional datasets align with the reality? 
After applying the graph method for inconsistency fixing, we can visualize the fixing results and compare them with the reality. The results of all models on the US State dataset are illustrated in Figure~\ref{fig:appendix-inconsistency-fixing}. We can see that subfigure \ref{app:gpt-4o} to \ref{app:r1} capture the reality well, while \ref{app:v2} to \ref{app:distill} is close to random.

\begin{figure*}[ht]
    \centering
    \begin{subfigure}[b]{0.3\textwidth}
        \includegraphics[width=\textwidth]{paper_figure/2D_map-reference.pdf}
        \caption{Reference(ground-truth)}
    \end{subfigure}
    
    \begin{subfigure}[b]{0.3\textwidth}
        \includegraphics[width=\textwidth]{paper_figure/2D_map-gpt-4o.pdf}
        \caption{GPT-4o}
        \label{app:gpt-4o}
    \end{subfigure}
    \begin{subfigure}[b]{0.3\textwidth}
        \includegraphics[width=\textwidth]{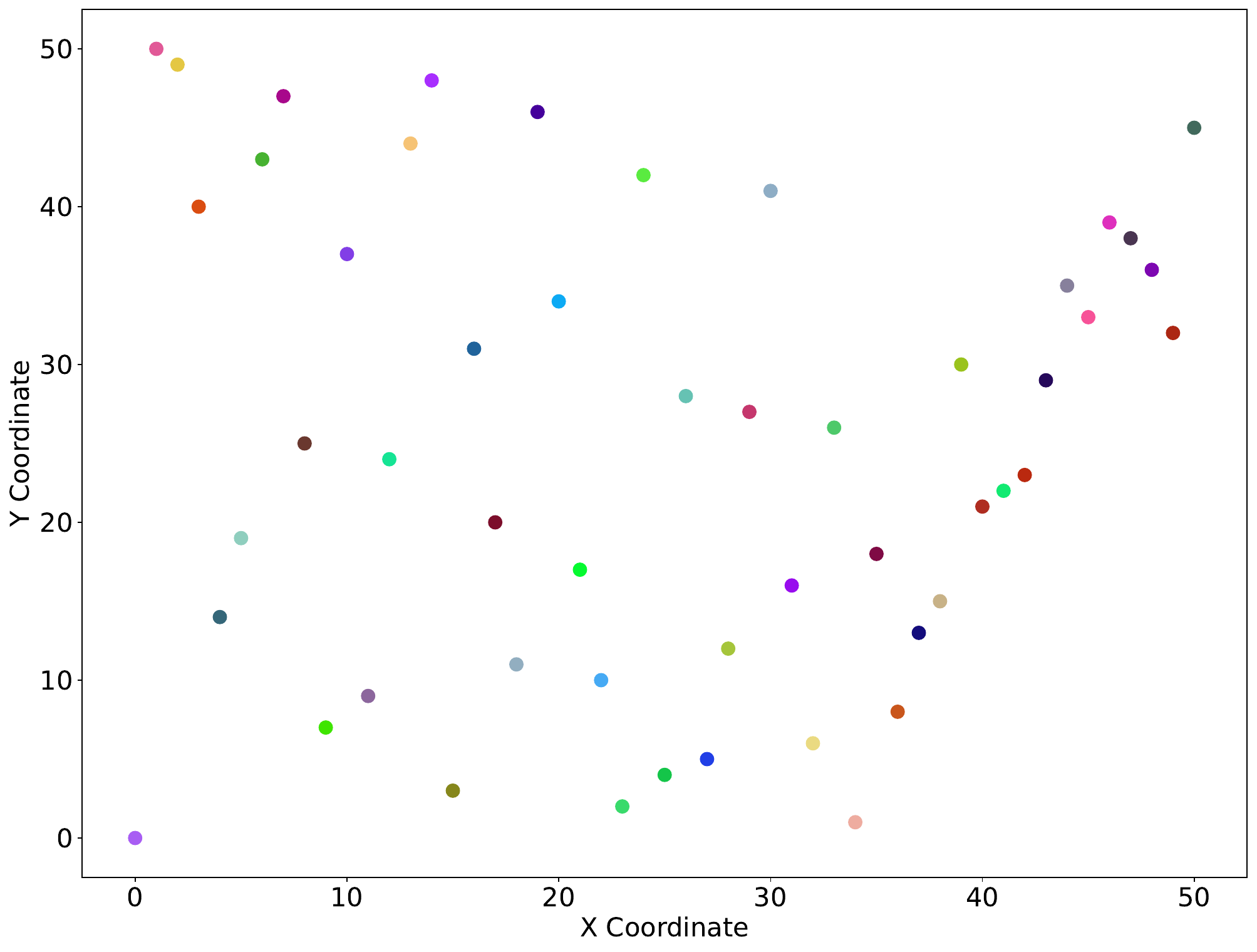}
        \caption{GPT-o1-mini}
    \end{subfigure}
            \begin{subfigure}[b]{0.3\textwidth}
        \includegraphics[width=\textwidth]{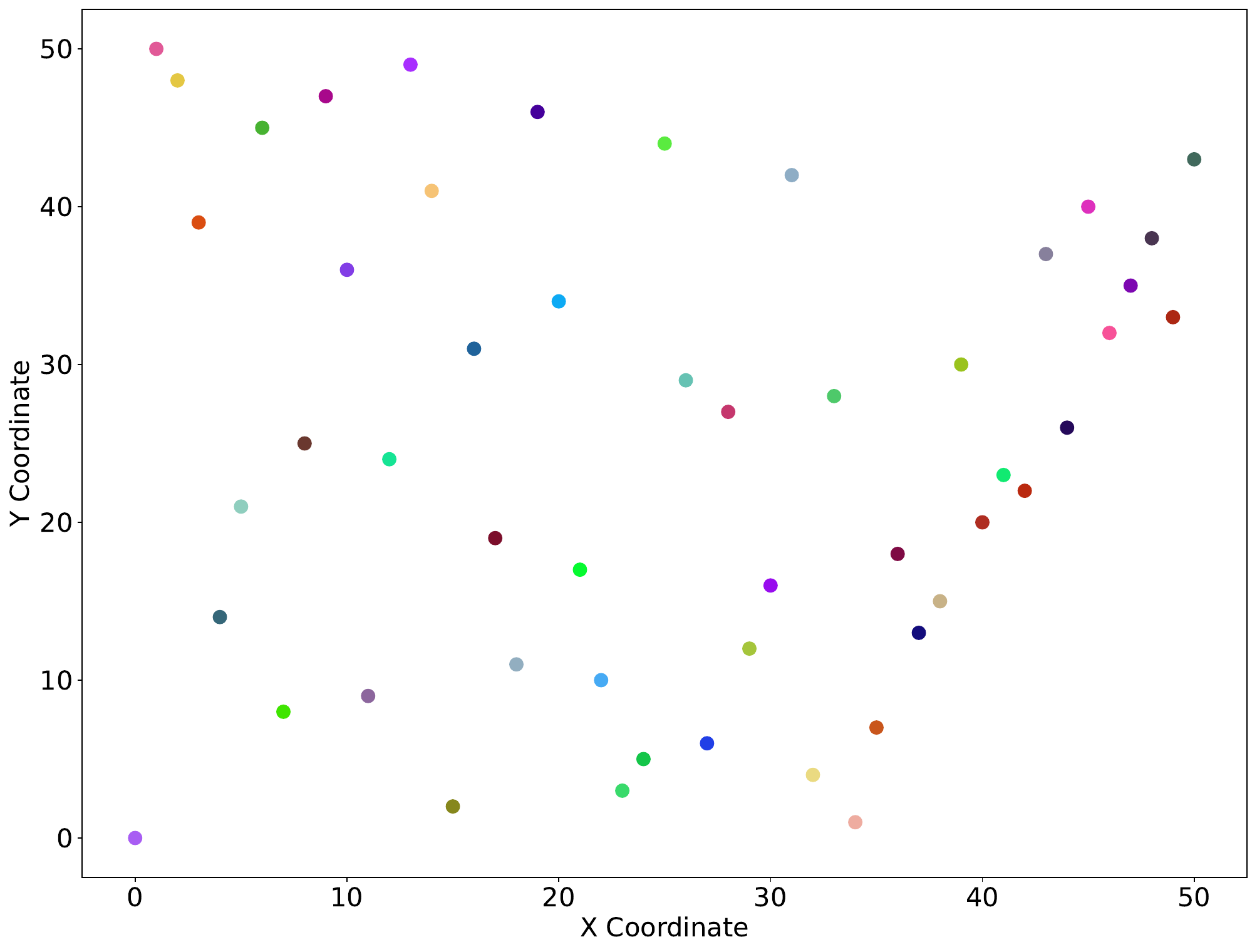}
        \caption{GPT-o4-mini}
    \end{subfigure}


    \begin{subfigure}[b]{0.3\textwidth}
        \includegraphics[width=\textwidth]{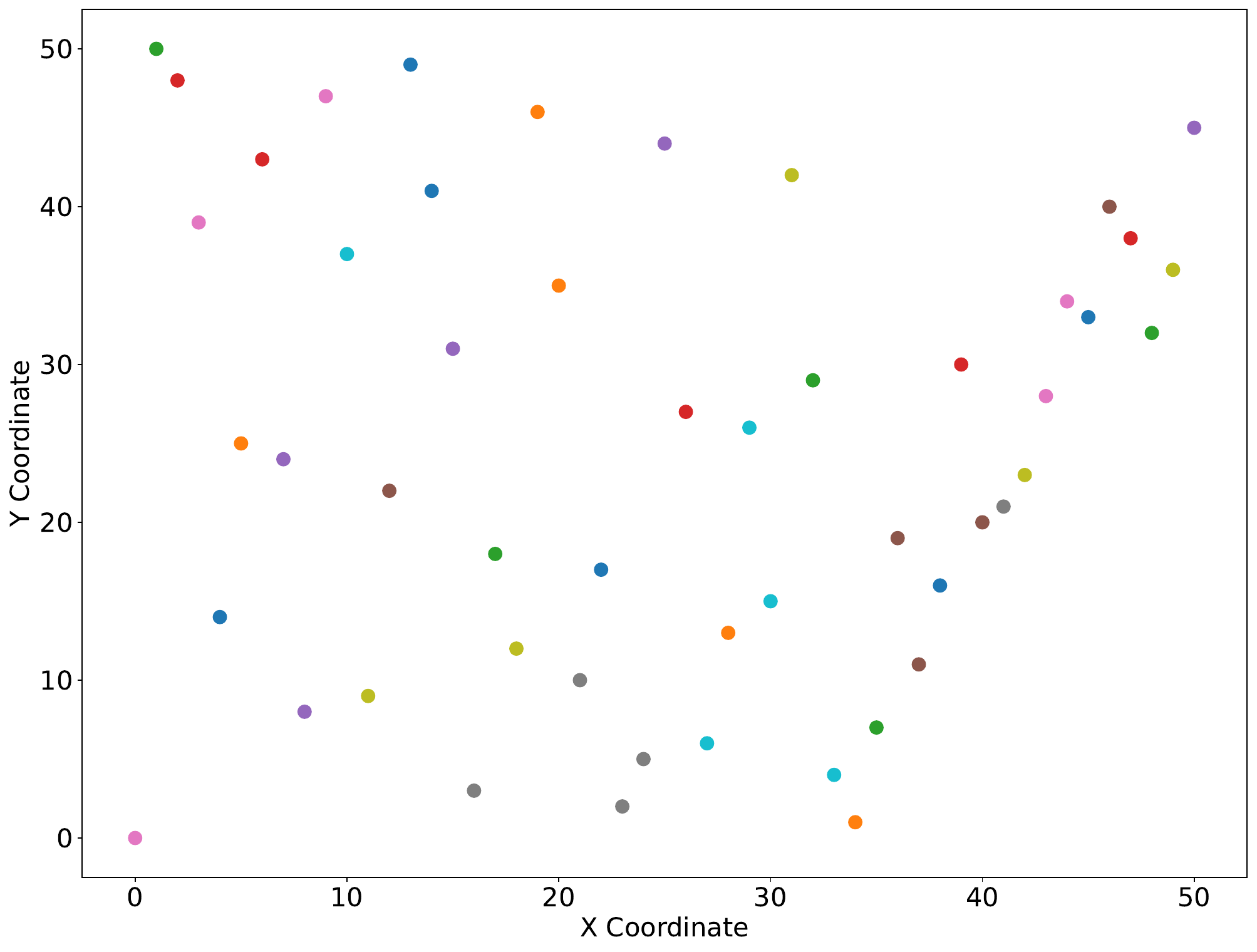}
        \caption{DeepSeek-V3}
    \end{subfigure}
    \begin{subfigure}[b]{0.3\textwidth}
        \includegraphics[width=\textwidth]{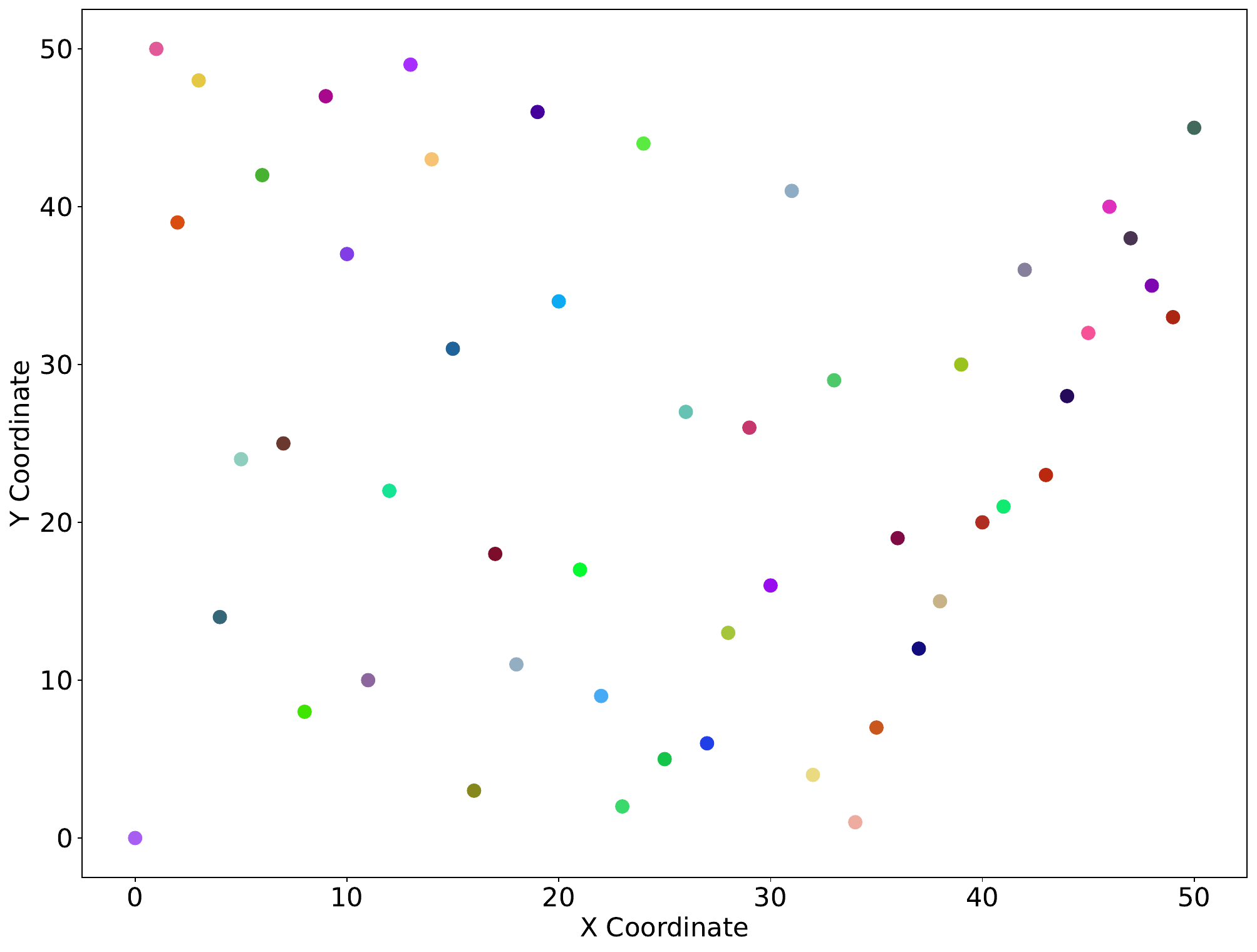}
        \caption{DeepSeek-R1}
        \label{app:r1}
    \end{subfigure}
    \begin{subfigure}[b]{0.3\textwidth}
        \includegraphics[width=\textwidth]{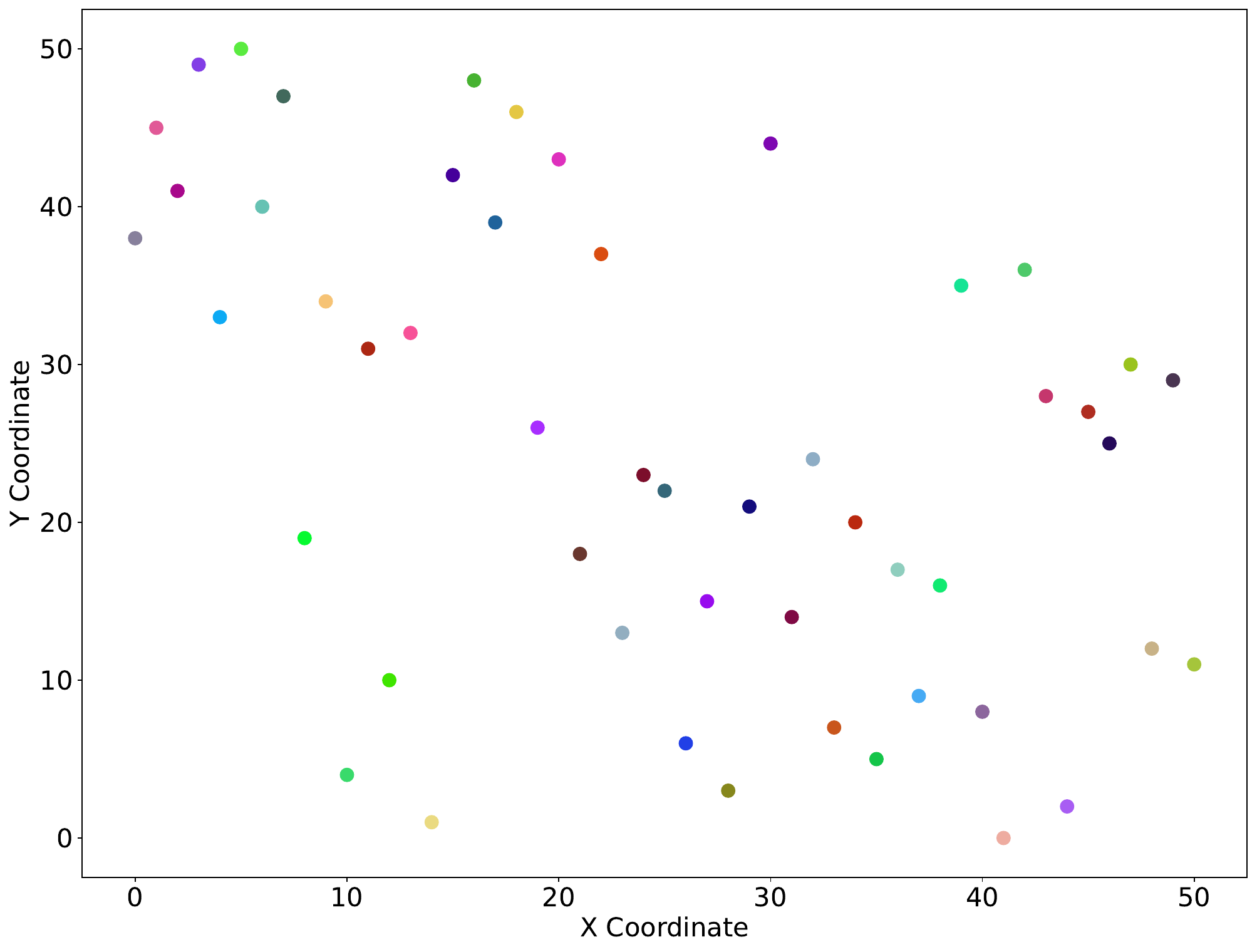}
        \caption{DeepSeek-V2}
        \label{app:v2}
    \end{subfigure}

    \begin{subfigure}[b]{0.3\textwidth}
        \includegraphics[width=\textwidth]{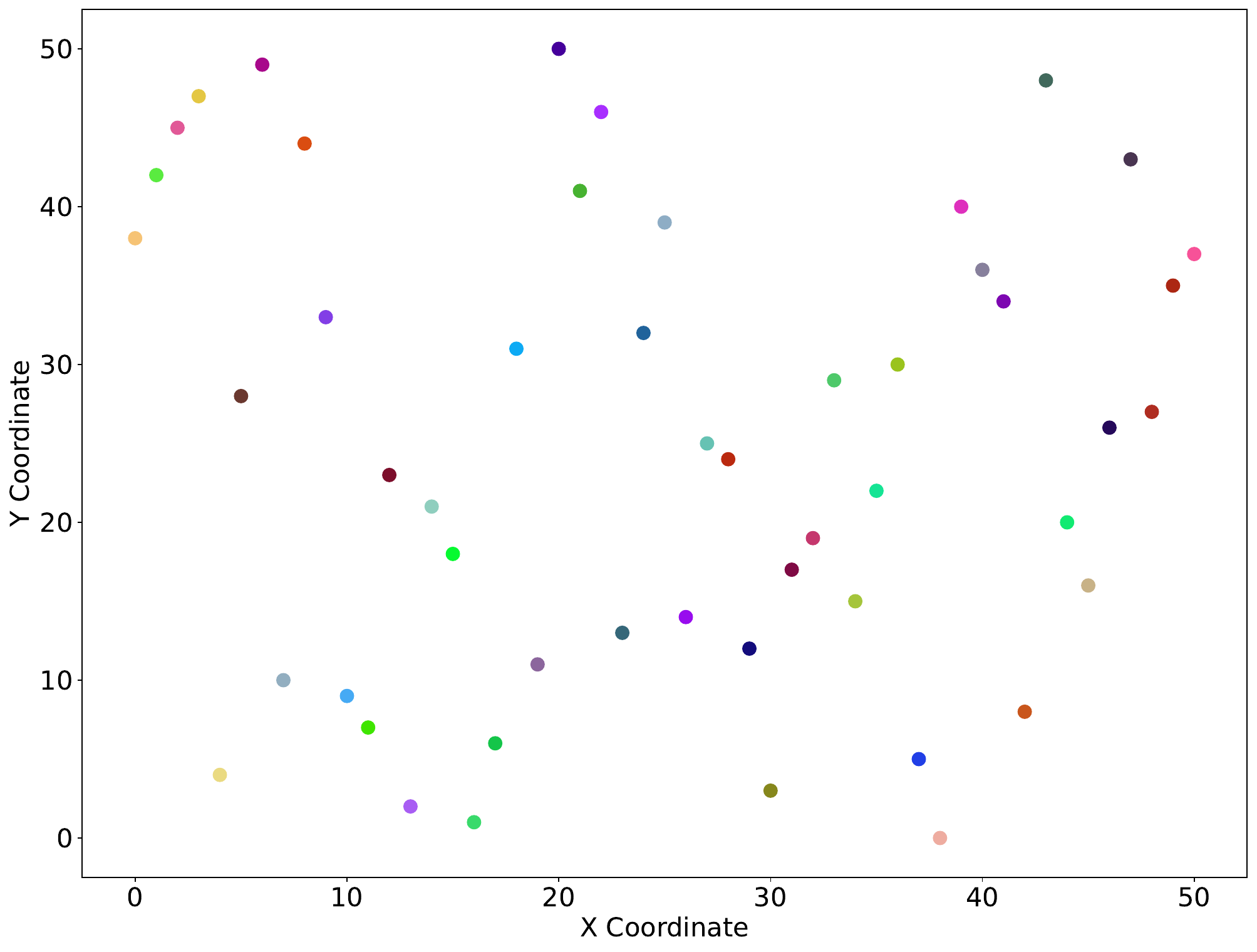}
        \caption{Llama3-8B-Instruct}
    \end{subfigure}
    \begin{subfigure}[b]{0.3\textwidth}
        \includegraphics[width=\textwidth]{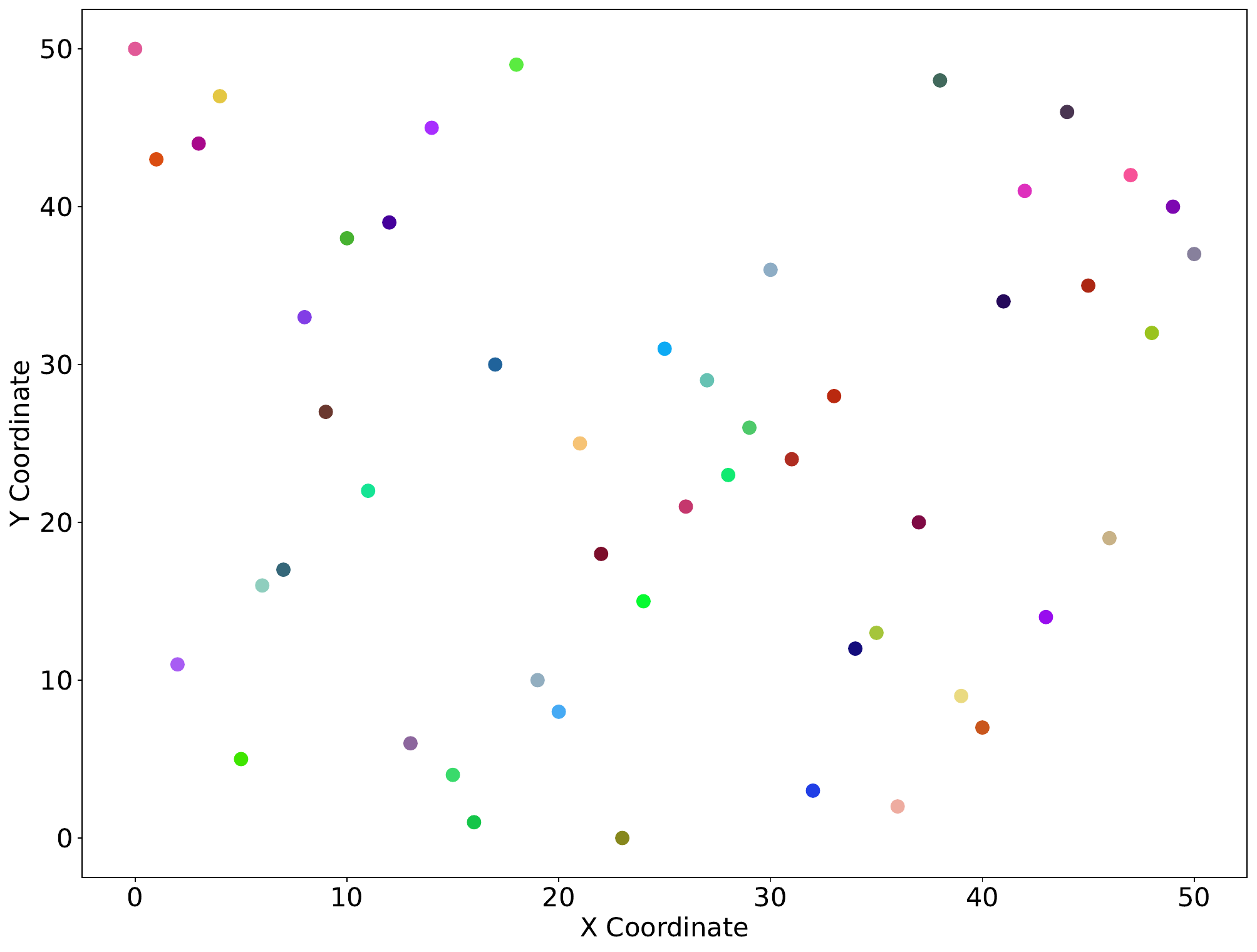}
        \caption{DeepSeek-R1-Distill-Qwen}
    \end{subfigure}
    \begin{subfigure}[b]{0.3\textwidth}
        \includegraphics[width=\textwidth]{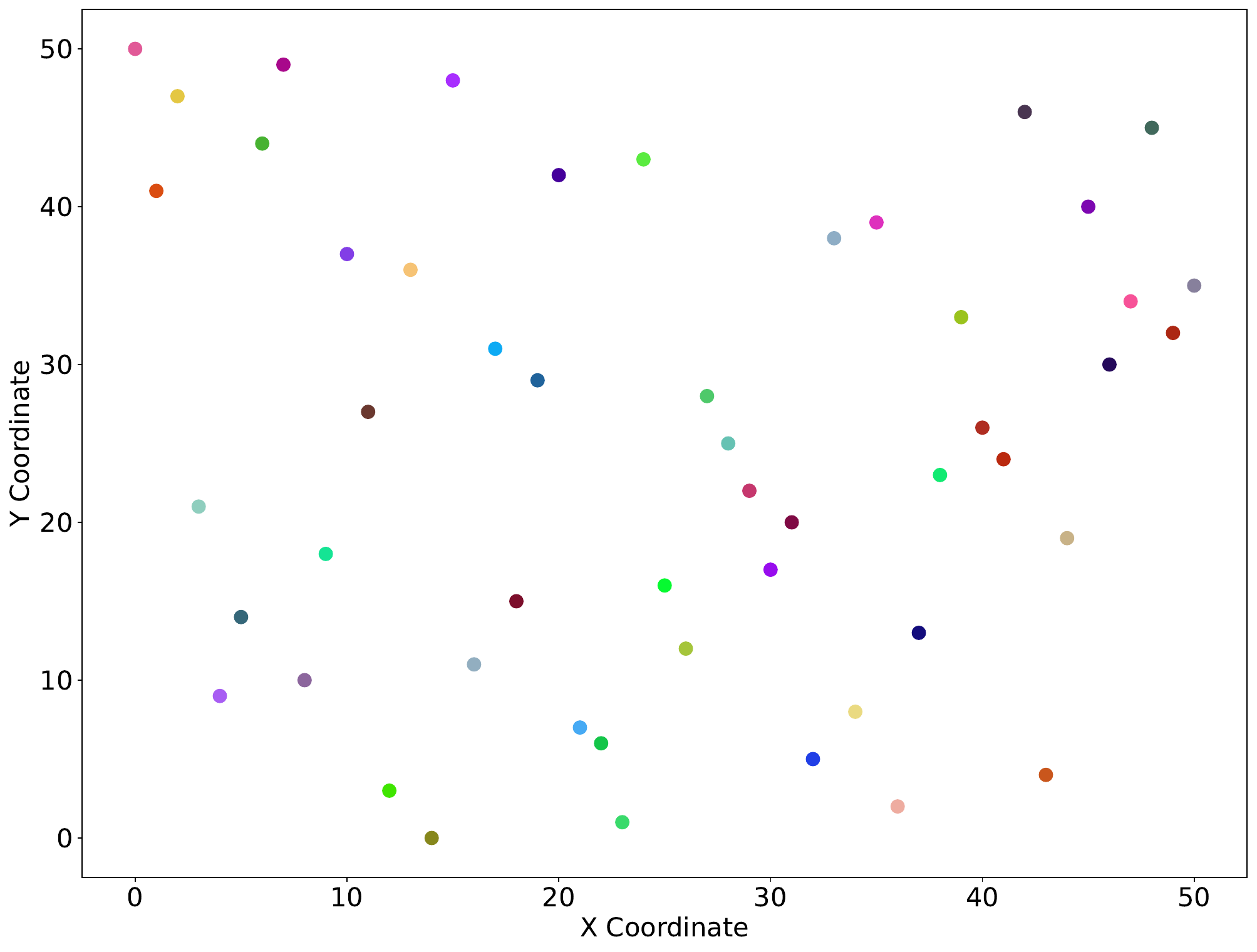}
        \caption{DeepSeek-R1-Distill-Llama}
        \label{app:distill}
    \end{subfigure}

    \begin{subfigure}[b]{0.3\textwidth}
        \includegraphics[width=\textwidth]{paper_figure/2D_map-Qwen2-7B-Chat.pdf}
        \caption{Qwen2-7B-Chat}
    \end{subfigure}
    \begin{subfigure}[b]{0.3\textwidth}
        \includegraphics[width=\textwidth]{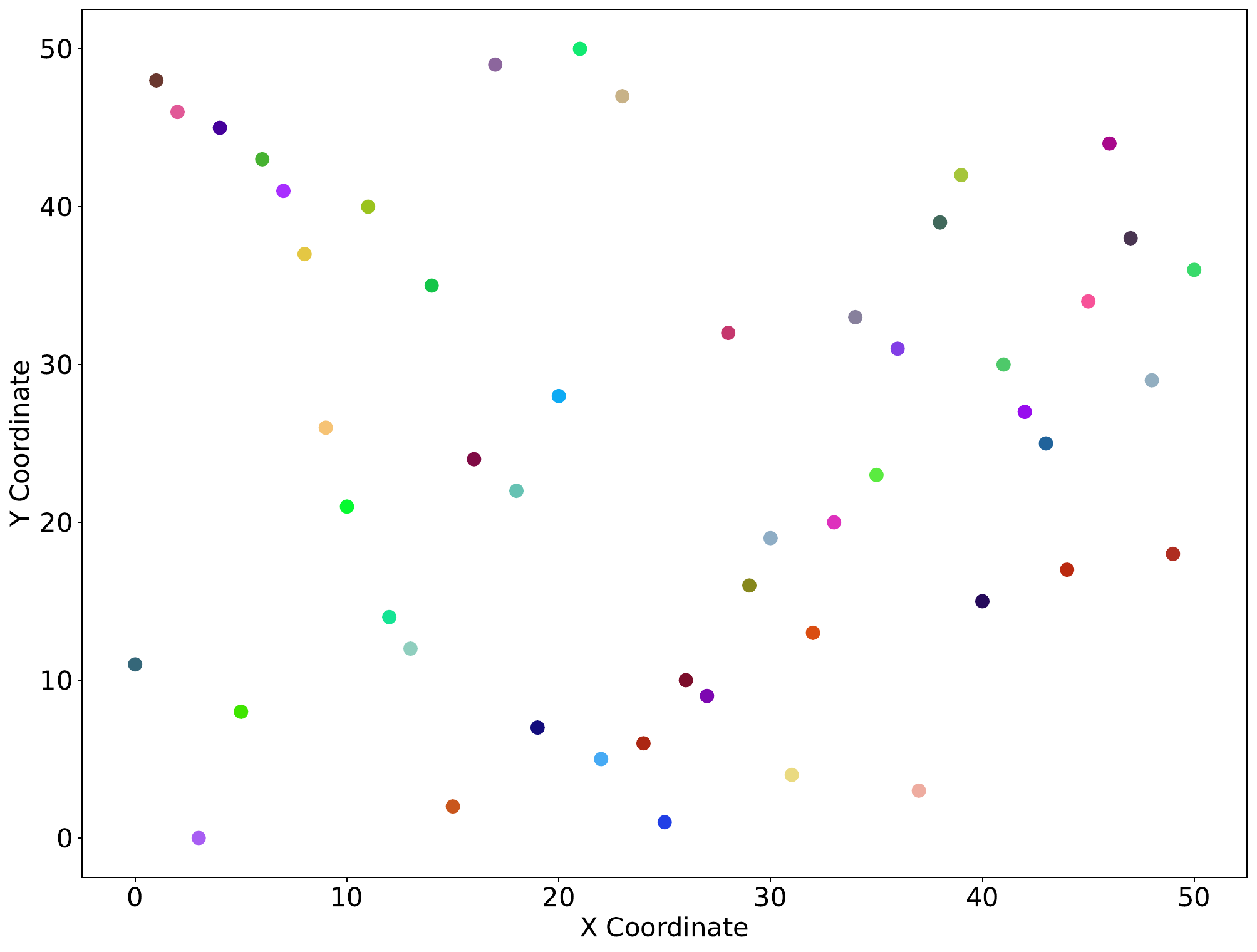}
        \caption{Qwen3-30B-A3B}
    \end{subfigure}
    \begin{subfigure}[b]{0.3\textwidth}
        \includegraphics[width=\textwidth]{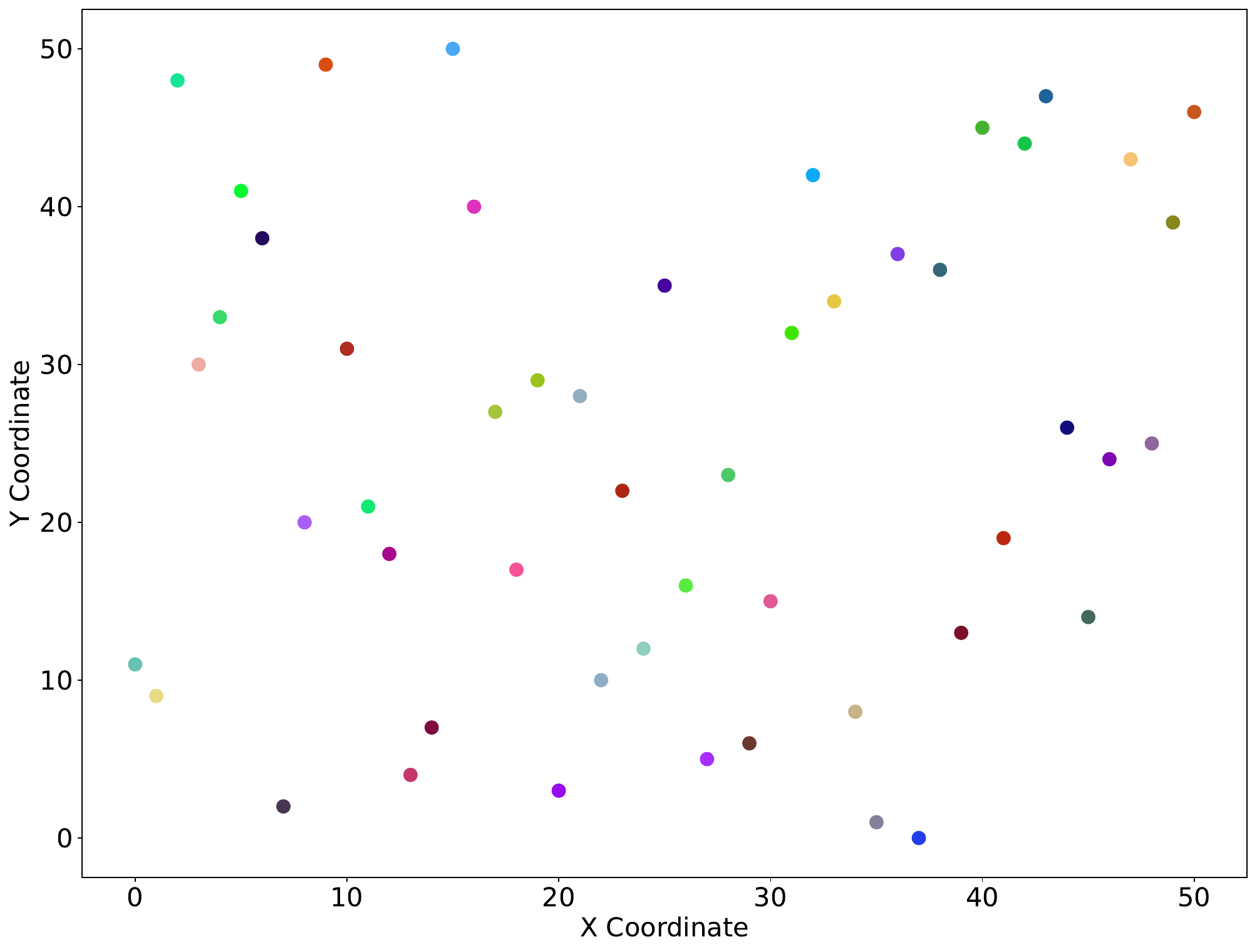}
        \caption{Qwen3-235B-A22B}
    \end{subfigure}

    \caption{2D Inconsistency Fixing Results}
    \label{fig:appendix-inconsistency-fixing}
\end{figure*}

\subsubsection{Reality Alignment: A Pilot Study}

\paragraph{Reality Alignment with Few Ground Truth}

Assuming the model has formed a self-consistent understanding of relationships among $n$ entities, but its understanding deviates from reality. How can we edit the model to align with reality while preserving internal self-consistency? 

The intuition stems from human learning, where inconsistencies are typically addressed by sampling a small subset of errors and correcting them. 
Our core approach involves verifying a few sampled critical nodes and edges against ground truth and leveraging transitivity to adjust other related edges accordingly, which is illustrated by Algorithm \ref{alg:alignment}. 

\begin{algorithm}[ht]
\caption{Iterative Consistent Alignment}
\label{alg:alignment}
\begin{algorithmic}
\STATE \textbf{Input:} Graph $G=(V,E)$, true ordering $\tau$
\STATE \textbf{Output:} Corrected graph $G'$
\WHILE{true}
    \STATE $\pi \leftarrow \text{NodeOrderingAlgorithm}(G)$\COMMENT{Current ordering}
    \STATE $R \leftarrow \text{FindReverseEdges}(G, \pi)$
    \IF{$R = \emptyset$}
        \STATE \textbf{break}
    \ENDIF
    
    \STATE $G \leftarrow \text{CorrectReverseEdges}(G, R, \tau)$
    \STATE $G \leftarrow \text{ApplyTransitivity}(G)$
    
    \STATE $v^* \leftarrow \arg\max_{v \in V} |\pi(v) - \tau(v)|$ \COMMENT{Find max displaced node}
    \STATE $d \leftarrow \pi(v^*) - \tau(v^*)$ \COMMENT{Displacement}
    
    \STATE $e_{fact} \leftarrow \text{SampleFactEdge}(v^*, d, \tau)$
    \STATE $G \leftarrow G \cup \{e_{fact}\}$ \COMMENT{Add fact edge}
    \STATE $G \leftarrow \text{ApplyTransitivity}(G)$
\ENDWHILE
\STATE \textbf{Return} $G$
\end{algorithmic}
\end{algorithm}

\begin{table}[ht]
    \centering
    \begin{adjustbox}{max width=\linewidth}
    \begin{tabular}{lccccccc}
        \toprule
        Model & x\_gt & x\_trans & x\_model & y\_gt & y\_trans & y\_model \\
        \midrule
        gpt-4o               & 14   & 3  & 1258 & 22  & 8   & 1244 \\
        o1-mini              & 27   & 8  & 1240 & 40  & 47  & 1188 \\
        DeepSeek-V3       & 6    & 1  & 1268 & 11  & 6   & 1255 \\
        \bottomrule
    \end{tabular}
    \end{adjustbox}
    \caption{Experimental results for the \texttt{us\_states} experiment when alignment with reality is achieved. Columns labeled \texttt{\_gt} represent the number of ground truth edges provided. \texttt{\_trans} denotes the number of correct edges deduced through transitivity (excluding ground truth edges), \texttt{\_model} captures the number of edges originally predicted by the models.}
    \label{tab:trans_align}
\end{table}

Table~\ref{tab:trans_align} shows the results of this algorithm on the US State dataset. The key metrics include the number of ground truth edges required (\texttt{gt}), the number of transitive edges deduced (\texttt{trans}), and the original model predictions (\texttt{model}). Most models achieved self-consistency after a single iteration, with only a small fraction requiring two iterations. This means that ``ground truth edges'' required by our method are about the number of ``the reversed edges initially detected'', which is almost the minimum of edges needed to fix the inconsistency.

This algorithm provides a potential way to align with reality with just a few ground truth knowledge given. However, in the experiments, we've found that it's hard to achieve complete self-consistency and reality alignment through ``simple LoRA finetuning on facts''. 

\paragraph{Can LLMs improve self-consistency via simple LoRA finetuning on facts?}

The smaller models are too far from reality, so we try to teach them all basic knowledge (time and positions for each object) in reality. 
The smaller models are given instructions to figure out and remember the exact time of each event, and the latitude and longitude of each place. After training on all datasets with LoRA, lr=5e-5, lr\_scheduler=cosine within 20 epochs and picking the best model (around epoch 4).

We've found that the factual accuracy of smaller models increased a lot after SFT (Table~\ref{tab:sft_know_acc}, but their inconsistency generally gets higher. As Table~\ref{tab:sft_incon} goes. The learned knowledge often fails to infer when questioned about relations. 

\begin{table}[ht]
    \centering
    \begin{adjustbox}{max width=\linewidth}
    \begin{tabular}{lccccc}
        \toprule
\textbf{Model} & \textbf{Art} & \textbf{Ancient Figure} & \textbf{Recent Figure} & \textbf{US City} & \textbf{US State} \\
        \midrule
        Llama-3-8B & 80 (+65) & 55 (+45) & 90 (+25) & 100 (+15) & 96.08 (+68.63) \\
        Qwen2-7B & 10 (+10) & 15 (+10) & 65 (+25) & 80 (+5) & 60.78 (+25.49) \\
        \bottomrule
    \end{tabular}
    \end{adjustbox}
    \caption{Factual accuracy (\%) After SFT (±compared with Before SFT). For 2D datasets, the scores are averaged for the x and y directions.}
    \label{tab:sft_know_acc}
\end{table}

\begin{table}[ht]
\centering
    \begin{adjustbox}{max width=\linewidth}
        \begin{tabular}{lccccc}
        \toprule
        \textbf{Model} & \textbf{Art} & \textbf{Ancient Figure} & \textbf{Recent Figure} & \textbf{US City} & \textbf{US State} \\
        \midrule
        Qwen2-7B & 40 (+23) & 44 (+8.7) & 44 (+12) & 42 (+16) & 42 (+13) \\
        Llama-3-8B & 15 (-1.2) & 20 (+1.2) & 18 (-0.5) & 33 (+3.2) & 33 (+0.2) \\
        \bottomrule
        \end{tabular}
        \end{adjustbox}
        \caption{Inconsistency Score (\%)  After SFT (±compared with Before SFT). For 2D datasets, the scores are averaged for the x and y directions.}
        \label{tab:sft_incon}
\end{table}

\end{document}